\documentclass[11pt]{article}
\pdfoutput=1

\usepackage{microtype}
\usepackage{graphicx}
\usepackage{subfigure}
\usepackage{booktabs} 
\usepackage{bm}
\usepackage{diagbox}
\usepackage{algorithm,algorithmic}
\usepackage[algo2e,ruled,vlined]{algorithm2e}

\usepackage{geometry}
\usepackage{amssymb}
\usepackage{graphicx}
\usepackage{url}
\usepackage{setspace}
\usepackage[pdftex,bookmarksnumbered,bookmarksopen,
colorlinks,citecolor=blue,linkcolor=blue,urlcolor=blue]{hyperref}
\usepackage{framed}
\usepackage{xcolor}
\usepackage{soul}
\usepackage{longtable}

\usepackage{times}
\usepackage{enumitem}
\usepackage{varwidth}
\usepackage{graphicx}
\usepackage{wrapfig}
\usepackage{enumerate}

\usepackage{amssymb}
\usepackage{graphicx}
\usepackage{url}
\usepackage{setspace}
\usepackage{framed}
\usepackage{xcolor}
\usepackage{soul}
\usepackage{longtable}

\usepackage{times}
\usepackage{enumitem}
\usepackage{varwidth}
\usepackage{graphicx}
\usepackage{wrapfig}
\usepackage{enumerate}

\usepackage[utf8]{inputenc} 
\usepackage[T1]{fontenc}    
\usepackage{url}            
\usepackage{booktabs}       
\usepackage{amsfonts}       
\usepackage{nicefrac}       
\usepackage{microtype}      

\usepackage{graphicx}
\usepackage{appendix}
\usepackage{amsmath,amsfonts,amsthm}
\usepackage{mdwlist}
\usepackage{xspace}
\usepackage{color}
\usepackage{mathrsfs}

\usepackage{booktabs}
\usepackage{comment}

\usepackage{multirow}

\newcommand{\sign}{\textup{\textsf{sign}}}

\newcommand{\adv}{\mathrm{rob}}
\newcommand{\nat}{\mathrm{nat}}

\newcommand{\boundary}{\mathrm{DB}}

\newcommand{\ERM}{\mathrm{ERM}}
\newcommand{\Appendix}[1]{the full version for}

\newtheorem{theorem}{Theorem}[section]
\newtheorem{lemma}[theorem]{Lemma}

\newtheorem{remark}{Remark}

\newtheorem{assumption}{Assumption}

\renewcommand{\u}{\bm{u}}
\renewcommand{\v}{\bm{v}}
\newcommand{\w}{\bm{w}}
\newcommand{\x}{\bm{x}}
\newcommand{\y}{\bm{y}}

\newcommand{\I}{\mathbf{I}}

\newcommand{\R}{\mathbb{R}}

\newcommand{\X}{\bm{X}}
\newcommand{\Y}{\bm{Y}}

\newcommand{\0}{\mathbf{0}}
\newcommand{\1}{\mathbf{1}}
\renewcommand{\comment}[1]{}

\newcommand{\cA}{\mathcal{A}}

\newcommand{\cD}{\mathcal{D}}
\newcommand{\cF}{\mathcal{F}}
\newcommand{\cG}{\mathcal{G}}

\newcommand{\cL}{\mathcal{L}}

\newcommand{\cN}{\mathcal{N}}

\newcommand{\cP}{\mathcal{P}}

\newcommand{\cR}{\mathcal{R}}
\newcommand{\cX}{\mathcal{X}}
\newcommand{\bbB}{\mathbb{B}}
\newcommand{\bbE}{\mathbb{E}}
\newcommand{\bbS}{\mathbb{S}}

\DeclareMathOperator*{\argmax}{argmax}
\DeclareMathOperator*{\argmin}{argmin}

\title{Theoretically Principled Trade-off between Robustness and Accuracy}

\author{
Hongyang Zhang\thanks{Part of this work was done while H. Z. was visiting Simons Institute
for the Theory of Computing.} \\ CMU \& TTIC \\ \small{hongyanz@cs.cmu.edu} \and
Yaodong Yu\thanks{Part of this work was done while Y. Y. was an intern at Petuum Inc.} \\ University of Virginia \\ \small{yy8ms@virginia.edu} \and
Jiantao Jiao \\ UC Berkeley \\
\small{jiantao@eecs.berkeley.edu} \and
Eric P. Xing \\ CMU \& Petuum Inc. \\
\small{epxing@cs.cmu.edu} \and
Laurent El Ghaoui \\ UC Berkeley \\
\small{elghaoui@berkeley.edu} \\ \and
Michael I. Jordan \\ UC Berkeley \\
\small{jordan@cs.berkeley.edu}
}

\date{}

\begin{document}
\maketitle

\begin{abstract}
We identify a trade-off between robustness and accuracy that serves as a guiding principle in the design of defenses against adversarial examples. Although this problem has been widely studied empirically, much remains unknown concerning the theory underlying this trade-off. In this work, we decompose the prediction error for adversarial examples (robust error) as the sum of the natural (classification) error and boundary error, and provide a differentiable upper bound using the theory of classification-calibrated loss, which is shown to be the tightest possible upper bound uniform over all probability distributions and measurable predictors. Inspired by our theoretical analysis, we also design a new defense method, TRADES, to trade adversarial robustness off against accuracy. Our proposed algorithm performs well experimentally in real-world datasets. The methodology is the foundation of our entry to the NeurIPS 2018 Adversarial Vision Challenge in which we won the 1st place out of \textasciitilde2,000 submissions, surpassing the runner-up approach by $11.41\%$ in terms of mean $\ell_2$ perturbation distance.
\end{abstract}

\section{Introduction}
In response to the vulnerability of deep neural networks to small perturbations around input data~\cite{szegedy2013intriguing}, adversarial defenses have been an imperative object of study in machine learning~\cite{huang2017adversarial}, computer vision~\cite{song2018pixeldefend,xie2017adversarial,meng2017magnet}, natural language processing~\cite{jia2017adversarial}, and many other domains. In machine learning, study of adversarial defenses has led to significant advances in understanding and defending against adversarial threat~\cite{he2017adversarial}. In computer vision and natural language processing, adversarial defenses serve as indispensable building blocks for a range of security-critical systems and applications, such as autonomous cars and speech recognition authorization. The problem of adversarial defenses can be stated as that of learning a classifier with high test accuracy on both natural and \emph{adversarial examples}. The adversarial example for a given labeled data $(\x,y)$ is a data point $\x'$ that causes a classifier $c$ to output a different label on $\x'$ than $y$, but is ``imperceptibly similar'' to $\x$. Given the difficulty of providing an operational definition of ``imperceptible similarity,''  adversarial examples typically come in the form of \emph{restricted attacks} such as $\epsilon$-bounded perturbations~\cite{szegedy2013intriguing}, or \emph{unrestricted attacks} such as adversarial rotations, translations, and deformations~\cite{brown2018unrestricted,engstrom2017rotation,gilmer2018motivating,xiao2018spatially,alaifari2018adef,zhang2018limitations}. The focus of this work is the former setting, though our framework can be generalized to the latter. 

Despite a large literature devoted to improving the robustness of deep-learning models, many fundamental questions remain unresolved. One of the most important questions is how to trade off adversarial robustness against natural accuracy. Statistically, robustness can be be at odds with accuracy~\cite{tsipras2018robustness}. This has led to an empirical line of work on adversarial defense that incorporates various kinds of assumptions~\cite{su2018robustness,kurakin2016adversarial}. On the theoretical front, methods such as \emph{relaxation based defenses}~\cite{kolter2017provable,raghunathan2018certified}  provide provable guarantees for adversarial robustness. They, however, ignore the performance of classifier on the non-adversarial examples, and thus leave open the theoretical treatment of the putative robustness/accuracy trade-off.

The problem of adversarial defense becomes more challenging when computational issues are considered. For example, the straightforward empirical risk minimization (ERM) formulation of robust classification involves minimizing the robust 0-1 loss
$
\max_{\x': \|\x'-\x\|\le\epsilon} \1\{c(\x')\not=y\},
$
a loss which is NP-hard to optimize even if $\epsilon = 0$ in general. Hence, it is natural to expect that some prior work on adversarial defense replaced the 0-1 loss $\1(\cdot)$ with a surrogate loss~\cite{madry2018towards,kurakin2016adversarial,pmlr-v80-uesato18a}. However, there is little theoretical guarantee on the tightness of this approximation. 


\begin{figure}
\centering
\includegraphics[scale=0.6]{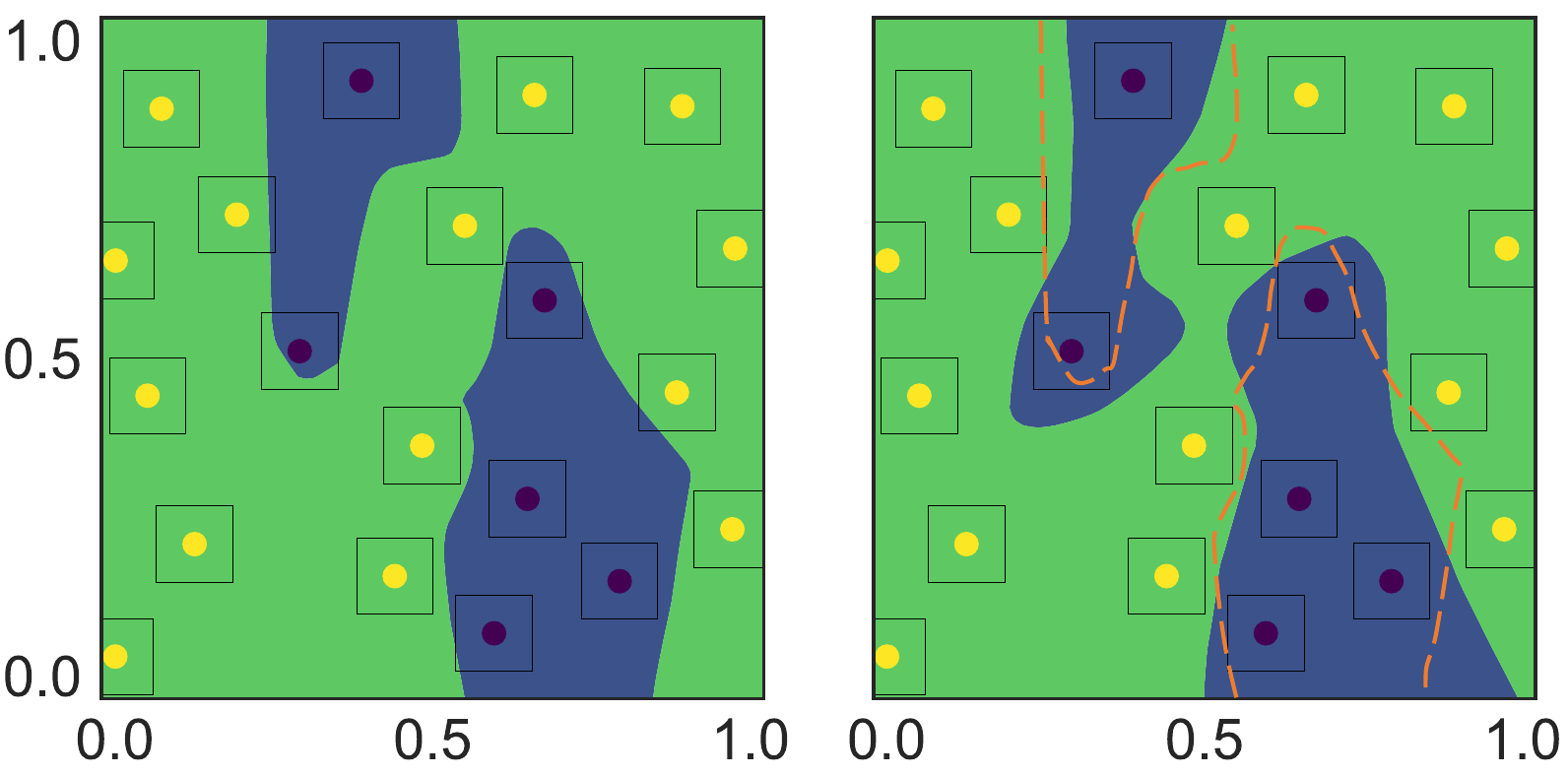}
\caption{\textbf{Left figure:} decision boundary learned by natural training method. \textbf{Right figure:} decision boundary learned by our adversarial training method, where the orange dotted line represents the decision boundary in the left figure. It shows that both methods achieve zero natural training error, while our adversarial training method achieves better robust training error than the natural training method.}
\label{figure: intro}
\end{figure}

\subsection{Our methodology and results}

We begin with an example that illustrates the trade-off between accuracy and adversarial robustness in Section~\ref{section: trade-off between natural and robust errors}. This phenomenon was first theoretically demonstrated by \cite{tsipras2018robustness}. We construct another toy example where the Bayes optimal classifier achieves \emph{natural error} $0\%$ and \emph{robust error} $100\%$, while the trivial all-one classifier achieves both \emph{natural error} and \emph{robust error} $50\%$ (Table~\ref{table: comparison of error}). While a large literature on the analysis of robust error in terms of generalization~\cite{schmidt2018adversarially,NIPS2018_7307,yin2018rademacher} and computational complexity~\cite{bubeck2018adversarial1,bubeck2018adversarial2}, in this work we focus on how to address the trade-off between the natural error and the robust error.

We show that the \emph{robust error} can in general be bounded tightly using two terms: one corresponds to the \emph{natural error} measured by a \emph{surrogate} loss function, and the other corresponds to how likely the input features are close to the $\epsilon$-extension of the decision boundary, termed as the \emph{boundary error}. We then minimize the differentiable upper bound. Our theoretical analysis naturally leads to a new formulation of adversarial defense which has several appealing properties; in particular, it inherits the benefits of scalability to large datasets exhibited by Tiny ImageNet, and the algorithm
achieves state-of-the-art performance on a range of benchmarks while providing theoretical guarantees. For example, while the defenses overviewed in \cite{athalye2018obfuscated} achieve robust accuracy no higher than \textasciitilde$47\%$ under white-box attacks, our method achieves robust accuracy as high as \textasciitilde$57\%$ in the same setting. The methodology is the foundation of our entry to the NeurIPS 2018 Adversarial Vision Challenge where we won first place out of \textasciitilde2,000 submissions, surpassing the runner-up approach by $11.41\%$ in terms of mean $\ell_2$ perturbation distance.



\subsection{Summary of contributions}
Our work tackles the problem of trading accuracy off against robustness and advances the state-of-the-art in multiple ways.
\vspace{-0.1cm}
\begin{itemize}
\item
Theoretically, we characterize the trade-off between accuracy and robustness for classification problems via decomposing the robust error as the sum of the natural error and the boundary error. We provide differentiable upper bounds on both terms using the theory of classification-calibrated loss, which are shown to be the tightest upper bounds uniform over all probability distributions and measurable predictors. 
\item
Algorithmically, inspired by our theoretical analysis, we propose a new formulation of adversarial defense, {TRADES}, as optimizing a regularized surrogate loss. The loss consists of two terms: the term of empirical risk minimization encourages the algorithm to maximize the natural accuracy, while the regularization term encourages the algorithm to push the decision boundary away from the data, so as to improve adversarial robustness (see Figure \ref{figure: intro}).
\item
Experimentally, we show that our proposed algorithm outperforms state-of-the-art methods under both black-box and white-box threat models. In particular, the methodology won the final round of the NeurIPS 2018 Adversarial Vision Challenge.
\end{itemize}

\section{Preliminaries}

We illustrate our methodology using the framework of binary classification, but it can be generalized to other settings as well. 

\subsection{Notations}

We will use \emph{bold capital} letters such as $\X$ and $\Y$ to represent random vector, \emph{bold lower-case} letters such as $\x$ and $\y$ to represent realization of random vector, \emph{capital} letters such as $X$ and $Y$ to represent random variable, and \emph{lower-case} letters such as $x$ and $y$ to represent realization of random variable. Specifically, we denote by $\x\in\cX$ the sample instance, and by $y\in\{-1,+1\}$ the label, where $\cX\subseteq \R^d$ indicates the instance space. $\sign(x)$ represents the sign of scalar $x$ with $\sign(0)=+1$. Denote by $f:\cX\rightarrow \R$ the \emph{score function} which maps an instance to a confidence value associated with being positive. It can be parametrized, e.g., by deep neural networks. The associated binary classifier is $\sign(f(\cdot))$. We will frequently use $\1\{\text{event}\}$, the 0-1 loss, to represent an indicator function that is $1$ if an event happens and $0$ otherwise. For norms, we denote by $\|\x\|$ a generic norm. Examples of norms include $\|\x\|_\infty$, the infinity norm of vector $\x$, and $\|\x\|_2$, the $\ell_2$ norm of vector $\x$. We use $\bbB(\x,\epsilon)$ to represent a neighborhood of $\x$: $\{\x'\in\cX:\|\x'-\x\|\le\epsilon\}$. For a given score function $f$, we denote by $\boundary(f)$ the decision boundary of $f$; that is, the set $\{\x\in\cX:f(\x)=0\}$. The set $\bbB(\boundary(f),\epsilon)$ denotes the neighborhood of the decision boundary of $f$: $\{\x\in\cX:\exists \x'\in\bbB(\x,\epsilon)\text{ s.t. } f(\x)f(\x')\le 0\}$. For a given function $\psi(\u)$, we denote by $\psi^*(\v):=\sup_{\u}\{\u^T\v-\psi(\u)\}$ the conjugate function of $\psi$, by $\psi^{**}$ the bi-conjugate, and by $\psi^{-1}$ the inverse function. We will frequently use $\phi(\cdot)$ to indicate the surrogate of 0-1 loss.

\subsection{Robust (classification) error}

In the setting of adversarial learning, we are given a set of instances $\x_1,...,\x_n\in\cX$ and labels $y_1,...,y_n\in\{-1,+1\}$. We assume that the data are sampled from an unknown distribution $(\X,Y)\sim\cD$. To characterize the robustness of a score function $f: \cX\rightarrow\R$, \cite{schmidt2018adversarially,NIPS2018_7307,bubeck2018adversarial1} defined \emph{robust (classification) error} under the threat model of bounded $\epsilon$ perturbation:
$
\cR_\adv(f):=\bbE_{(\X,Y)\sim\cD}\1\{\exists \X'\hspace{-0.1cm}\in\hspace{-0.1cm}\bbB(\X,\epsilon) \text{ s.t. } f(\X')Y\le 0\}.
$
This is in sharp contrast to the standard measure of  classifier performance---the \emph{natural (classification) error}
$
    \cR_\nat(f):=\bbE_{(\X,Y)\sim\cD}\1\{f(\X)Y\le 0\}.
$
We note that the two errors satisfy $\cR_\adv(f)\ge\cR_\nat(f)$ for all $f$; the robust error is equal to the natural error when $\epsilon=0$. 

\subsection{Boundary error}

We introduce the \emph{boundary error} defined as 
$
    \mathcal{R}_{\text{bdy}}(f) :=\bbE_{(\X,Y)\sim\cD}\1\{\X \in \mathbb{B}(\text{DB}(f),\epsilon), f(\X)Y> 0\}.
$
We have the following decomposition of $\cR_\adv(f)$:
\begin{align} \label{eqn.fundamentalequation}
    \cR_\adv(f) = \cR_\nat(f) +\cR_{\text{bdy}}(f). 
\end{align}

\subsection{Trade-off between natural and robust errors}
\label{section: trade-off between natural and robust errors}

Our study is motivated by the trade-off between natural and robust errors.
\cite{tsipras2018robustness} theoretically showed that training models to be robust may lead to a reduction of standard accuracy by constructing a toy example. To illustrate the phenomenon, we provide another toy example here.

\medskip
\noindent{\textbf{Example.}}
Consider the case $(X,Y)\sim\cD$, where the marginal distribution over the instance space is a uniform distribution over $[0,1]$, and for $k=0,1,...,\lceil\frac{1}{2\epsilon}-1\rceil$,
\begin{equation}
\label{equ: counterexample}
\begin{split}
\eta(x)&:=\Pr(Y=1|X=x)\\
&=
\begin{cases}
0, & x\in[2k\epsilon,(2k+1)\epsilon),\\
1, & x\in((2k+1)\epsilon,(2k+2)\epsilon].
\end{cases}
\end{split}
\end{equation}
See Figure \ref{figure: tradeoff} for the visualization of $\eta(x)$. We consider two classifiers: a) the Bayes optimal classifier $\sign(2\eta(x)-1)$; b) the all-one classifier which always outputs ``positive.'' Table \ref{table: comparison of error} displays the trade-off between natural and robust errors: the minimal natural error is achieved by the Bayes optimal classifier with large robust error, while the optimal robust error is achieved by the all-one classifier with large natural error.

\begin{figure}
\centering
\includegraphics[scale=0.8]{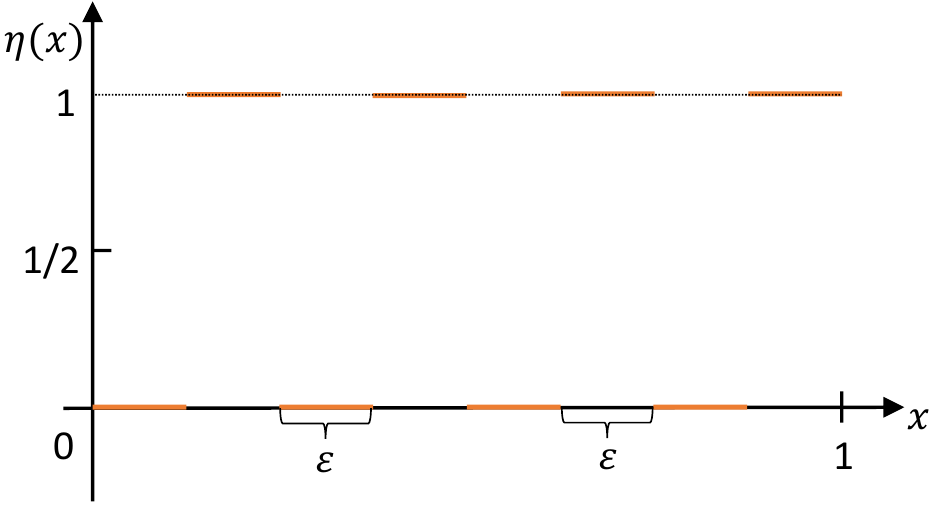}
\caption{Counterexample given by Eqn.~\eqref{equ: counterexample}.}
\label{figure: tradeoff}
\end{figure}

\begin{table}
\caption{Comparisons of natural and robust errors of Bayes optimal classifier and all-one classifier in example \eqref{equ: counterexample}. The Bayes optimal classifier has the optimal natural error while the all-one classifier has the optimal robust error.}
\label{table: comparison of error}
\centering
\begin{tabular}{c|c|c}%
\hline
& Bayes Optimal Classifier & All-One Classifier
\\
\hline
$\cR_\nat$ & 0 (optimal) & 1/2\\
\hline
$\cR_{\text{bdy}}$ & 1 & 0 \\
\hline
$\cR_\adv$ & 1 & 1/2 (optimal)\\
\hline
\end{tabular}
\vspace{-0.4cm}
\end{table}

\medskip
\noindent{\textbf{Our goal.}}
In practice, one may prefer to trade-off between robustness and accuracy by introducing weights in (\ref{eqn.fundamentalequation}) to bias more towards the natural error or the boundary error. Noting that both the natural error and the boundary error involve 0-1 loss functions, our goal is to devise \emph{tight} differentiable upper bounds on both of these terms. Towards this goal, we utilize the theory of classification-calibrated loss.




\subsection{Classification-calibrated surrogate loss}

\noindent{\textbf{Definition.}}
Minimization of the 0-1 loss in the natural and robust errors is computationally intractable and the demands of computational efficiency have led researchers to focus on minimization of a tractable \emph{surrogate loss}, $\cR_\phi(f):=\bbE_{(\X,Y)\sim\cD}\phi(f(\X)Y)$.  We then need to find quantitative relationships between the excess errors associated with $\phi$ and those associated with 0–1 loss. We make a weak assumption on $\phi$: it is \emph{classification-calibrated}~\cite{bartlett2006convexity}.  Formally, for $\eta\in[0,1]$, define the \emph{conditional $\phi$-risk} by
\begin{equation*}
H(\eta):=\inf_{\alpha\in\R} C_\eta(\alpha):=\inf_{\alpha\in\R} \left(\eta\phi(\alpha)+(1-\eta)\phi(-\alpha)\right),
\end{equation*}
and define $H^-(\eta):=\inf_{\alpha(2\eta-1)\le 0} C_\eta(\alpha)$. The classification-calibrated condition requires that imposing the
constraint that $\alpha$ has an inconsistent sign with the Bayes decision rule $\sign(2\eta-1)$ leads to a strictly larger $\phi$-risk:
\begin{assumption}[Classification-Calibrated Loss]
\label{assumption: classification-calibrated}
We assume that the surrogate loss $\phi$ is classification-calibrated, meaning that for any $\eta\not=1/2$, $H^-(\eta)>H(\eta)$.
\end{assumption}
We argue that Assumption \ref{assumption: classification-calibrated} is indispensable for classification problems, since without it the Bayes optimal classifier cannot be the minimizer of the $\phi$-risk.
Examples of classification-calibrated loss include hinge loss, sigmoid loss, exponential loss, logistic loss, and many others (see Table \ref{table: examples of classification-calibrated loss}).

\medskip
\noindent{\textbf{Properties.}} Classification-calibrated loss has many structural properties that one can exploit. We begin by introducing a functional transform of classification-calibrated loss $\phi$ which was proposed by \cite{bartlett2006convexity}.  Define the function $\psi:[0,1]\rightarrow[0,\infty)$ by $\psi=\widetilde\psi^{**}$, where
$\widetilde\psi(\theta):=H^-\left(\frac{1+\theta}{2}\right)-H\left(\frac{1+\theta}{2}\right)$. Indeed, the function $\psi(\theta)$ is the largest convex lower bound on $H^-\left(\frac{1+\theta}{2}\right)-H\left(\frac{1+\theta}{2}\right)$. The value $H^-\left(\frac{1+\theta}{2}\right)-H\left(\frac{1+\theta}{2}\right)$ characterizes how close the surrogate loss $\phi$ is to the class of non-classification-calibrated losses.

Below we state useful properties of the $\psi$-transform. We will frequently use the function $\psi$ to bound $\cR_\adv(f)-\cR_\nat^*$.

\begin{lemma}[\cite{bartlett2006convexity}]
Under Assumption \ref{assumption: classification-calibrated}, the function $\psi$ has the following properties: $\psi$ is non-decreasing, continuous, convex on $[0,1]$ and $\psi(0) = 0$. 
\end{lemma}

\begin{table}
\caption{Examples of classification-calibrated loss $\phi$ and associated $\psi$-transform. Here $\psi_{\mathsf{log}}(\theta) =  \frac{1}{2}(1-\theta) \log_2(1-\theta) + \frac{1}{2}(1+\theta)\log_2(1+\theta)$.}
\label{table: examples of classification-calibrated loss}
\centering
\begin{tabular}{c|cc}%
\hline
Loss & $\phi(\alpha)$ & $\psi(\theta)$
\\
\hline
Hinge & $\max\{1-\alpha,0\}$ & $\theta$\\
Sigmoid & $1-\tanh(\alpha)$ & $\theta$\\
Exponential & $\exp(-\alpha)$ & $1-\sqrt{1-\theta^2}$\\
Logistic & $\log_2(1+\exp(-\alpha))$ & $ \psi_{\mathsf{log}}(\theta)$\\
\hline
\end{tabular}
\end{table}

\section{Relating 0-1 loss to Surrogate Loss}
\vspace{-0.05cm}

In this section, we present our main theoretical
contributions for binary classification and compare our results with  prior literature. Binary classification problems have received significant attention in recent years as many competitions evaluate the performance of robust models on binary classification problems~\cite{brown2018unrestricted}. We defer the discussion of multi-class problems to Section \ref{section: algorithmic design for adversarial defenses}.

\vspace{-0.2cm}
\subsection{Upper bound}
\label{section: upper bound}

Our analysis leads to a guarantee on the performance of surrogate loss minimization. Intuitively, by Eqn. \eqref{eqn.fundamentalequation}, $\cR_{\adv}(f)-\cR_{\nat}^*=\cR_\nat(f)-\cR_{\nat}^*+\cR_{\text{bdy}}(f)\le \psi^{-1}(\cR_\phi(f)-\cR_\phi^*)+\cR_{\text{bdy}}(f)$, where the last inequality holds because we choose $\phi$ as a classification-calibrated loss~\cite{bartlett2006convexity}. This leads to the following result.

\begin{theorem}
\label{theorem: surrogate function}
Let $\cR_\phi(f):=\bbE\phi(f(\X)Y)$ and $\cR_\phi^*:=\min_f \cR_\phi(f)$. Under Assumption \ref{assumption: classification-calibrated}, for any non-negative loss function $\phi$ such that $\phi(0)\ge 1$, any measurable $f:\cX\rightarrow \R$, any probability distribution on $\cX\times\{\pm 1\}$, and any $\lambda>0$, we have\footnote{We study the population form of the risk functions, and mention that by incorporating the generalization theory for classification-calibrated losses~\cite{bartlett2006convexity} one can extend the analysis to finite samples. We leave this analysis for future research.}
\begin{equation*}\vspace{-0.3cm}
\begin{split}
\cR_\adv(f)-\cR_\nat^*&\le \psi^{-1}(\cR_\phi(f)\hspace{-0.1cm}-\hspace{-0.1cm}\cR_\phi^*)\hspace{-0.1cm}+\hspace{-0.1cm}\Pr[\X\hspace{-0.1cm}\in\hspace{-0.1cm}\bbB(\boundary(f),\epsilon),f(\X)Y>0]\\
&\le \psi^{-1}(\cR_\phi(f)\hspace{-0.1cm}-\hspace{-0.1cm}\cR_\phi^*)+\bbE \max_{\X'\in\bbB(\X,\epsilon)}\phi(f(\X')f(\X)/\lambda).
\end{split}
\end{equation*}
\end{theorem}

\comment{
\begin{remark}
When $g$ is a linear classifier going through origin and $X$ is drawn from isotropic log-concave distribution, we know that~\cite{awasthi2016learning}
\begin{equation}
\Pr[X\in\bbB(\boundary(g),\epsilon)]=\Theta(\epsilon).
\end{equation}
So linear classifier is robust. The reason why deep neural network is vulnerable to adversarial attack is that $\Pr[X\in\bbB(\boundary(g),\epsilon)]$ is large as the decision boundary is very complicated.
\end{remark}
}

\vspace{-0.05cm}
\medskip
\noindent{\textbf{Quantity governing model robustness.}}
Our result provides a formal justification for the existence of adversarial examples: learning models are vulnerable to small adversarial attacks because the probability that data lie around the decision boundary of the model, $\Pr[\X\in \bbB(\boundary(f),\epsilon),f(\X)Y>0]$, is large. As a result, small perturbations may move the data point to the wrong side of the decision boundary, leading to weak robustness of classification models.


\comment{
On the other hand, consider the example in Section \ref{section: trade-off between natural and robust errors}. According to \cite{bartlett2006convexity}, surrogate loss minimization leads to the Bayes optimal classifier $\widehat f$ given sufficient model capacity. In this case, $\Pr[\X\in \bbB(\boundary(f),\epsilon),c_0(\X)=Y]=1$. We notice that $\cR_\adv(\widehat f)-\cR_\nat^*$ is as large as $1$ as well (see Table \ref{table: comparison of error}), demonstrating a strong trade-off between accuracy and robustness.
}

\comment{
Apart from it, another quantity governing the model robustness is the label noise level $\Pr[\sign(2\eta(\X)-1)\not =Y]$. Low noise level, surprisingly, implies that the difference between $\cR_\adv(\widehat f)$ and $\cR_\nat^*$ is large, provided that $\widehat f$ is the solution of surrogate loss minimization $\min_f \cR_\phi(f)$. Reconsider the example in Section \ref{section: trade-off between natural and robust errors}. According to \cite{bartlett2006convexity}, surrogate loss minimization leads to the Bayes optimal classifier $\widehat f$ given sufficient model capacity.  There is no label noise. However, $\cR_\adv(\widehat f)-\cR_\nat^*$ is as large as $1$ (see Table \ref{table: comparison of error}), demonstrating a strong trade-off between accuracy and robustness.
}

\comment{
\noindent{\textbf{Trade-off Regarding Model Capacity.}}
Theorem \ref{theorem: surrogate function} implies a potential trade-off regarding the model capacity $|\cF|$:
\begin{equation*}
\begin{split}
&\quad\cR_\adv(f)-\cR_\nat^*\\
&\le \psi^{-1}\Big(\cR_\phi(f)-\min_{f\in\cF} \cR_\phi(f)+\underbrace{\min_{f\in\cF} \cR_\phi(f)-\cR_\phi^*}_{\text{term (a)},\ \downarrow\text{ when }|\cF|\ \uparrow}\Big)\\
&\quad+\underbrace{\Pr[\X\in\bbB(\boundary(f),\epsilon)]}_{\text{term (b), }\uparrow\text{ when }|\cF|\ \uparrow}.
\end{split}
\end{equation*}
We note that term (a), the approximation term, shrinks as the model capacity increases. Term (b), however, may grow with the increase of model capacity (see Appendix \ref{section: adversarial vulnerability under log-concave distributions} for an example).
}

\vspace{-0.1cm}
\subsection{Lower bound}
\label{section: lower bound}

We now establish a lower bound on $\cR_\adv(f)-\cR_\nat^*$. Our lower bound matches our analysis of the upper bound in Section \ref{section: upper bound} up to an arbitrarily small constant.

\begin{theorem}
\label{theorem: tightness of surrogate loss}
Suppose that $|\cX|\ge 2$. Under Assumption \ref{assumption: classification-calibrated}, for any non-negative loss function $\phi$ such that $\phi(x)\rightarrow 0$ as $x\rightarrow +\infty$, any $\xi>0$, and any $\theta\in[0,1]$, there exists a probability distribution on $\cX\times \{\pm 1\}$, a function $f:\R^d\rightarrow\R$, and a regularization parameter $\lambda>0$ such that
$\cR_\adv(f)-\cR_\nat^*=\theta$
and
\begin{equation*}
\begin{split}
\psi\Big(\theta-\bbE &\max_{\X'\in\bbB(\X,\epsilon)}\phi(f(\X')f(\X)/\lambda)\Big)\le \cR_\phi(f)-\cR_\phi^*\le \psi\left(\theta-\bbE \max_{\X'\in\bbB(\X,\epsilon)}\phi(f(\X')f(\X)/\lambda)\right)+\xi.
\end{split}
\end{equation*}
\end{theorem}

Theorem \ref{theorem: tightness of surrogate loss} demonstrates that in the presence of extra conditions on the loss function, i.e., $\lim_{x\rightarrow +\infty} \phi(x)=0$, the upper bound in Section \ref{section: upper bound} is tight. The condition holds for all the losses in Table \ref{table: examples of classification-calibrated loss}.

\section{Algorithmic Design for Defenses}
\label{section: algorithmic design for adversarial defenses}

\noindent{\textbf{Optimization.}}
Theorems \ref{theorem: surrogate function} and \ref{theorem: tightness of surrogate loss} shed light on algorithmic designs of adversarial defenses.
In order to minimize $\cR_\adv(f)-\cR_\nat^*$, the theorems suggest minimizing\footnote{For simplicity of implementation, we do not use the function $\psi^{-1}$ and rely on $\lambda$ to approximately reflect the effect of $\psi^{-1}$, the trade-off between the natural error and the boundary error, and the tight approximation of the boundary error using the corresponding surrogate loss function. }
\begin{equation}
\label{equ: new loss}
\min_{f} \bbE \Big\{\underbrace{\phi(f(\X)Y)}_{\text{for accuracy}}+\underbrace{\max_{\X'\in\bbB(\X,\epsilon)} \phi(f(\X)f(\X')/\lambda)}_{\text{regularization for robustness}}\Big\}.
\end{equation}
We name our method \textbf{TRADES} (TRadeoff-inspired Adversarial DEfense via Surrogate-loss minimization).

\medskip
\noindent{\textbf{Intuition behind the optimization.}}
Problem \eqref{equ: new loss}  captures the trade-off between the natural and robust errors: the first term in \eqref{equ: new loss} encourages the natural error to be optimized by minimizing the ``difference'' between $f(\X)$ and $Y$, while the second regularization term encourages the output to be smooth, that is, it pushes the decision boundary of classifier away from the sample instances via minimizing the ``difference'' between the prediction of natural example $f(\X)$ and that of adversarial example $f(\X')$. This is conceptually consistent with the argument that smoothness is an indispensable property of robust models~\cite{cisse2017parseval}. The tuning parameter $\lambda$ plays a critical role on balancing the importance of natural and robust errors. To see how the $\lambda$ affects the solution in the example of Section \ref{section: trade-off between natural and robust errors}, problem \eqref{equ: new loss} tends to the Bayes optimal classifier when $\lambda\rightarrow +\infty$, and tends to the all-one classifier when $\lambda\rightarrow 0$.

\medskip
\noindent{\textbf{Comparisons with prior work.}}
We compare our approach with several related lines of research in the prior literature.  One of the best known algorithms for adversarial defense is based on  \emph{robust optimization}~\cite{madry2018towards,kolter2017provable,wong1805scaling,raghunathan2018certified,raghunathan2018semidefinite}. Most results in this direction involve algorithms that approximately minimize
\begin{equation}
\label{equ: old loss}
\min_{f} \bbE \left\{ \max_{\X'\in\bbB(\X,\epsilon)}\phi(f(\X')Y)\right\},
\end{equation}
where the objective function in problem \eqref{equ: old loss} serves as an upper bound of the robust error $\cR_\adv(f)$. In complex problem domains, however, this objective function might not be tight as an upper bound of the robust error, and may not capture the trade-off between natural and robust errors.

A related line of research is adversarial training by regularization~\cite{miyato2018virtual,kurakin2016adversarial,ross2017improving,zheng2016improving}. There are several key differences between the results in this paper and those of \cite{kurakin2016adversarial,ross2017improving,zheng2016improving}. Firstly, the optimization formulations are different. In the previous works, the regularization term either measures the ``difference'' between $f(\X')$ and $Y$~\cite{kurakin2016adversarial}, or its gradient~\cite{ross2017improving}. In contrast, our regularization term measures the ``difference'' between $f(\X)$ and $f(\X')$. While \cite{zheng2016improving} generated the adversarial example $\X'$ by adding random Gaussian noise to $\X$, our method simulates the adversarial example by solving the inner maximization problem in Eqn. \eqref{equ: new loss}. Secondly, we note that the losses in \cite{miyato2018virtual,kurakin2016adversarial,ross2017improving,zheng2016improving} lack of theoretical guarantees. Our loss, with the presence of the second term in problem \eqref{equ: new loss}, makes our theoretical analysis significantly more subtle. Moreover, our algorithm takes the same computational resources as~\cite{kurakin2016adversarial}, which makes our method scalable to large-scale datasets. We defer the experimental comparisons of various regularization based methods to Table \ref{table: icml best paper defense}.

\medskip
\noindent{\textbf{Differences with Adversarial Logit Pairing.}} We also compare TRADES with Adversarial Logit Pairing (ALP)~\cite{kannan2018adversarial,engstrom2018evaluating}. The algorithm of ALP works as follows: given a fixed network $f$ in each round, the algorithm firstly generates an adversarial example $\X'$ by solving
$\argmax_{\X'\in\bbB(\X,\epsilon)} \phi(f(\X')Y)$; ALP then updates the network parameter by solving a minimization problem
\begin{equation*}
\min_f \bbE \left\{\alpha \phi(f(\X')Y)+(1-\alpha)\phi(f(\X)Y)+\|f(\X)-f(\X')\|_{2}/\lambda\right\},
\end{equation*}
where $0\le\alpha\le 1$ is a regularization parameter; the algorithm finally repeats the above-mentioned procedure until it converges.
We note that there are fundamental differences between TRADES and ALP. While ALP simulates adversarial example $\X'$ by the FGSM$^{k}$ attack, TRADES simulates $\X'$ by solving $\argmax_{\X'\in\bbB(\X,\epsilon)} \phi(f(\X)f(\X')/\lambda)$. Moreover, while ALP uses the $\ell_2$ loss between $f(\X)$ and $f(\X')$ to regularize the training procedure without theoretical guarantees, TRADES uses the classification-calibrated loss according to Theorems \ref{theorem: surrogate function} and \ref{theorem: tightness of surrogate loss}.

\medskip
\noindent{\textbf{Heuristic algorithm.}}
In response to the optimization formulation \eqref{equ: new loss}, we use two heuristics to achieve more general defenses: a) extending to multi-class problems by involving multi-class calibrated loss; b) approximately solving the minimax problem via alternating gradient descent.
For multi-class problems, a surrogate loss is \emph{calibrated} if minimizers of the surrogate risk are also minimizers of the 0-1
risk~\cite{pires2016multiclass}. Examples of multi-class calibrated loss include cross-entropy loss. Algorithmically, we extend problem \eqref{equ: new loss} to the case of multi-class classifications by replacing $\phi$ with a multi-class calibrated loss $\cL(\cdot,\cdot)$:
\begin{equation}
\label{equ: multi-class loss}
\min_{f} \bbE \left\{\cL(f(\X),\Y)+\hspace{-0.2cm}\max_{\X'\in\bbB(\X,\epsilon)} \cL(f(\X),f(\X'))/\lambda\right\},
\end{equation}
where $f(\X)$ is the output vector of learning model (with softmax operator in the top layer for the cross-entropy loss $\cL(\cdot,\cdot)$), $\Y$ is the label-indicator vector, and $\lambda>0$ is the regularization parameter. One can also exchange $f(\X)$ and $f(\X')$ in the second term of \eqref{equ: multi-class loss}. The pseudocode of adversarial training procedure, which aims at minimizing the empirical form of problem \eqref{equ: multi-class loss}, is displayed in Algorithm \ref{algorithm: adversarial training of network}.

\begin{algorithm}[t]
\caption{Adversarial training by TRADES}
\label{algorithm: adversarial training of network}
\begin{algorithmic}[1]
\STATE {\bfseries Input:} Step sizes $\eta_1$ and $\eta_2$, batch size $m$, number of iterations $K$ in inner optimization, network architecture parametrized by $\theta$
\STATE {\bfseries Output:} Robust network $f_\theta$
\STATE{Randomly initialize network $f_\theta$, or initialize network with pre-trained configuration}
\REPEAT
\STATE{Read mini-batch $B=\{\x_1,...,\x_m\}$ from training set}
\FOR{$i=1,...,m$ (in parallel)}
\STATE{$\x_i'\leftarrow \x_i+0.001\cdot\cN(\0,\I)$, where $\cN(\0,\I)$ is the Gaussian distribution with zero mean and identity variance\label{dddd}}
\FOR{$k=1,...,K$}
\STATE{$\x_i'\leftarrow \Pi_{\bbB(\x_i,\epsilon)}(\eta_1\sign(\nabla_{\x_i'} \cL(f_\theta(\x_i),f_\theta(\x_i')))+\x_i')$, where $\Pi$ is the projection operator}
\ENDFOR
\ENDFOR
\STATE{$\theta\leftarrow \theta-\eta_2 \sum_{i=1}^m \nabla_\theta[\cL(f_\theta(\x_i),\y_i)+\cL(f_\theta(\x_i),f_\theta(\x_i'))/\lambda]/m$}
\UNTIL{training converged}
\end{algorithmic}
\end{algorithm}

The key ingredient of the algorithm is to approximately solve the linearization of inner maximization in problem \eqref{equ: multi-class loss} by the \emph{projected gradient descent} (see Step 7). We note that $\x_i$ is a global minimizer with zero gradient to the objective function $g(\x'):=\cL(f(\x_{i}),f(\x'))$ in the inner problem. Therefore, we initialize $\x_i'$ by adding a small, random perturbation around $\x_i$ in Step 5 to start the inner optimizer. More exhaustive approximations of the inner maximization problem in terms of either optimization formulations or solvers would lead to better defense performance.

\medskip
\noindent{\textbf{Semi-supervised learning.}}
We note that TRADES problem \eqref{equ: multi-class loss} can be straightforwardly applied to the semi-supervised learning framework, as the second term in problem \eqref{equ: multi-class loss} does not depend on the label $\Y$. Therefore, with more unlabeled data points, one can approximate the second term (in the expectation form) better by the empirical loss minimization. There are many interesting recent works which explore the benefits of invloving unlabeled data~\cite{carmon2019unlabeled,stanforth2019labels,zhai2019adversarially}.

\medskip
\noindent{\textbf{Acceleration.}}
Adversarial training is typically more than 10x slower than natural training. To resolve this issue for TRADES, \cite{shafahi2019adversarial,zhang2019you} proposed new algorithms to solve problem \eqref{equ: multi-class loss} at negligible additional cost compared to natural training.

\comment{
\section{Justification of Minimax Formulation}
\begin{theorem}
For hinge loss function $\phi$, any probability distribution on $\cX\times\{\pm 1\}$, and any measurable $g:\cX\rightarrow \R$ such that $\lambda(g)\ge \epsilon^{-1}\Pr[X\in\bbB(\boundary(g),\epsilon),\sign(\eta(X)-1/2)=Y]$, we have
\begin{equation}
R_\adv(g)-R_\nat^*\le \bbE\max_{X'\in\bbS(X,\epsilon)}\phi(Yg(X'))-R_\phi^*,
\end{equation}
where $\lambda(g)$ on the mapping $g:\R^d\rightarrow\R^t$ is defined as $\lambda(g):=\min_{x\in\cX} \|g(x)\|/\|x\|$.
\end{theorem}

\begin{proof}
When $\phi$ is the hinge loss function, $\psi$ is an identity mapping~\cite{bartlett2006convexity}.
So by Theorem \ref{theorem: surrogate function}, we have
\begin{equation}
R_\adv(g)-R_\nat^*\le R_\phi(g)-R_\phi^*+\Pr[X\in\bbB(\boundary(g),\epsilon),\sign(\eta(X)-1/2)=Y].
\end{equation}
Thus it suffices to show that
\begin{equation}
R_\phi(g)+\Pr[X\in\bbB(\boundary(g),\epsilon),\sign(\eta(X)-1/2)=Y]\le \bbE\max_{X'\in\bbS(X,\epsilon)}\phi(Yg(X')).
\end{equation}
To this end, we note that
\begin{equation}
\begin{split}
&\bbE\max_{X'\in\bbS(X,\epsilon)}\phi(Yg(X'))-R_\phi(g)\\
&=\bbE\left[\max_{X'\in\bbS(X,\epsilon)}\phi(Yg(X'))-\phi(Yg(X))\right]\\
&=\int \max_{X'\in\bbB(X,\epsilon)}\left\{\eta(X)[\phi(g(X'))-\phi(g(X))]+(1-\eta(X))[\phi(-g(X'))-\phi(-g(X))]\right\}d\Pr(X)\\
&\ge \int_{X\in\bbB(\boundary(g),\epsilon)} \max_{X'\in\bbB(X,\epsilon)}\left\{\eta(X)[\phi(g(X'))-\phi(g(X))]+(1-\eta(X))[\phi(-g(X'))-\phi(-g(X))]\right\}d\Pr(X)\\
&\ge\int_{X\in\bbB(\boundary(g),\epsilon)} \max_{X'\in\bbS(X,\epsilon)}\phi(Yg(X'))-\phi(Yg(X))d\Pr(X)\\
&\ge \bbE [\lambda(g)\|X'-X\|]\quad\text{(by the definition of $\lambda(g)$)}\\
&= \epsilon\lambda(g)\quad\text{(because $X'\in\bbS(X,\epsilon)$)}\\
&\ge \Pr[X\in\bbB(\boundary(g),\epsilon),\sign(\eta(X)-1/2)=Y],
\end{split}
\end{equation}
as desired.
\end{proof}

}

\comment{
\section{Lower Bound}

\begin{theorem}
Suppose that $|\cX|\ge 2$. For any non-negative loss function $\phi$, any $\epsilon>0$ and any $\theta\in[0,1]$, there is a probability distribution on $\cX\times\{\pm 1\}$ and a function $g:\cX\rightarrow \R$ such that
\begin{equation}
R_\adv(g)-R_\nat^*=\theta
\end{equation}
and
\begin{equation}
\psi(\theta-P)\le R_\phi(g)-R_\phi^*\le\psi(\theta-P)+\epsilon,
\end{equation}
where $P=\Pr[X\in\bbB(\boundary(g),\epsilon),\sign(\eta(X)-1/2)=Y]$.
\end{theorem}
}

\comment{
\section{Excess Risk}

Let $g^*(x)$ be the optimal classifier that minimizes $R_\adv(g)$. By Eqn. \eqref{equ: R_adv}, we have
\begin{equation}
\begin{split}
&R_\adv(g)-R_\adv(g^*)\\
&=\Pr(X\in \bbB(\boundary(g),\epsilon))+\int_{\bbB(\boundary(g),\epsilon)^\perp} [\1\{\sign(g(x))=-1\}(2\eta(x)-1)+(1-\eta(x))]d \Pr\nolimits_X(x)\\
&\quad -\Pr(X\in \bbB(\boundary(g^*),\epsilon))-\int_{\bbB(\boundary(g^*),\epsilon)^\perp} [\1\{\sign(g^*(x))=-1\}(2\eta(x)-1)+(1-\eta(x))]d \Pr\nolimits_X(x)\\
&\le \Pr(X\in \bbB(\boundary(g),\epsilon))+\int_{\bbB(\boundary(g),\epsilon)^\perp} [\1\{\sign(g(x))=-1\}(2\eta(x)-1)+(1-\eta(x))]d \Pr\nolimits_X(x)\\
&\quad -\Pr(X\in \bbB(\boundary(g^*),\epsilon),X\in\bbB(\boundary(g),\epsilon)^\perp)\\
&\quad-\Pr(X\in \bbB(\boundary(g^*),\epsilon),X\in\bbB(\boundary(g),\epsilon))\\
&\quad -\int_{\bbB(\boundary(g^*),\epsilon)^\perp\cap \bbB(\boundary(g),\epsilon)^\perp} [\1\{\sign(g^*(x))=-1\}(2\eta(x)-1)+(1-\eta(x))]d \Pr\nolimits_X(x)\\
&\le \Pr(X\in\bbB(\boundary(g),\epsilon))-\Pr(X\in\bbB(\boundary(g^*),\epsilon),X\in\bbB(\boundary(g),\epsilon))\\
&\quad+\int_{\bbB(\boundary(g),\epsilon)^\perp} [\1\{\sign(g(x))=-1\}(2\eta(x)-1)+(1-\eta(x))]d \Pr\nolimits_X(x)\\
&\quad -\int_{\bbB(\boundary(g),\epsilon)^\perp} [\1\{\sign(g^*(x))=-1\}(2\eta(x)-1)+(1-\eta(x))]d \Pr\nolimits_X(x)\\
&\le\Pr(X\in\bbB(\boundary(g),\epsilon))-\Pr(X\in\bbB(\boundary(g^*),\epsilon)^\perp)\\
&\quad+\bbE[\1\{\sign(g(X))\not=\sign(g^*(X)),X\in\bbB(\boundary(g),\epsilon)^\perp\}|2\eta(X)-1|]
\end{split}
\end{equation}
}

\comment{
\section{Uniform Convergence}

\begin{lemma}[Theorem 26.5, \cite{shalev2014understanding}]
Assume that for all $\x$ and $g\in\cG$ we have that $|\phi(g(\x))|\le c$. Let $\phi$ be $\rho$-Lipschitz. Then with probability of at least $1-\delta$,
\begin{equation}
R_\phi(\ERM_\cG(S))-\inf_{g\in\cG} R_\phi(g)\le 2\rho\cR(\cG\circ S)+5c\sqrt{\frac{2\ln(8/\delta)}{m}},
\end{equation}
where $S$ is the sample set with $|S|=m$ and $\cR(\cG\circ S)$ is the Rademacher complexity.
\end{lemma}

Plugging this into Theorem \ref{theorem: surrogate function}, we have
\begin{theorem}
Assume that for all $\x$ and $g\in\cG$ we have that $|\phi(g(\x))|\le c$. For any non-negative loss function $\phi$, any measurable $g:\cX\rightarrow \R$, and any probability distribution on $\cX\times\{\pm 1\}$, with probability of at least $1-\delta$,
\begin{equation}
\begin{split}
&R_\adv(\ERM_\cG(S))-R_\nat^*\\
&\le\psi^{-1}\left(2\rho\cR(\cG\circ S)+5c\sqrt{\frac{2\ln(8/\delta)}{m}}+\inf_{g\in\cG} R_\phi(g)-R_\phi^*\right)\\
&\quad +\Pr[X\in\bbB(\boundary(\ERM_\cG(S)),\epsilon),\sign(\eta(X)-1/2)=Y],
\end{split}
\end{equation}
where the function $\psi$ is defined as in \cite{bartlett2006convexity} which is independent of $g$.
\end{theorem}
}

\vspace{-0.1cm}
\section{Experimental Results}\label{sec:experiments}
In this section, we verify the effectiveness of TRADES by numerical experiments. We denote by $\cA_\adv(f) = 1- \cR_\adv(f)$ the robust accuracy, and by $\cA_\nat(f) = 1- \cR_\nat(f)$ the natural accuracy on test dataset. We release our code and trained models at \url{https://github.com/yaodongyu/TRADES}.

\vspace{-0.2cm}
\subsection{Optimality of Theorem \ref{theorem: surrogate function}}

We verify the tightness of the established upper bound in Theorem \ref{theorem: surrogate function} for binary classification problem on MNIST dataset. The negative examples are `1' and the positive examples are `3'. Here we use a Convolutional Neural Network (CNN) with two convolutional layers, followed by two fully-connected layers. The output size of the last layer is 1. To learn the robust classifier, we minimize the regularized surrogate loss in Eqn. \eqref{equ: new loss}, and use the hinge loss in Table \ref{table: examples of classification-calibrated loss} as the surrogate loss $\phi$, where the associated $\psi$-transform is $\psi(\theta) = \theta$.

To verify the tightness of our upper bound, we calculate the left hand side in Theorem \ref{theorem: surrogate function}, i.e., \vspace{-0.2cm}$$\Delta_{\text{LHS}}=\cR_\adv(f)-\cR_\nat^*,$$ \vspace{-0.2cm}
and the right hand side, i.e., $$\Delta_{\text{RHS}} = (\cR_\phi(f)-\cR_\phi^*)+\bbE \max_{\X'\in\bbB(\X,\epsilon)}\phi(f(\X')f(\X)/\lambda).$$ 
As we cannot have access to the unknown distribution $\cD$, we approximate the above expectation terms by test dataset. We first use natural training method to train a classifier so as to approximately estimate 
$\cR_\nat^*$ and $\cR_\phi^*$, where we find that the naturally trained classifier can achieve natural error $\cR_\nat^*=0\%$, and loss value $\cR_\phi^*=0.0$ for the binary classification problem. Next,  we optimize problem~\eqref{equ: new loss} to train a robust classifier $f$. We take perturbation $\epsilon = 0.1$, number of iterations $K = 20$ and run $30$ epochs on the training dataset. Finally, to approximate the second term in $\Delta_{\text{RHS}}$, we use FGSM$^k$ (white-box) attack (a.k.a. PGD attack)~\cite{kurakin2016adversarial} with $20$ iterations to approximately calculate the worst-case perturbed data $\bm{X}^{\prime}$. 

\begin{table}
	\caption{Theoretical verification on the optimality of Theorem \ref{theorem: surrogate function}.}
	\label{table: theoretical upper bound tightness verification}
	\centering
	\begin{tabular}{c|ccc}%
		\hline
		$\lambda$ & $\cA_\adv(f)$  $(\%)$ & $\cR_\phi(f)$ & $\Delta = \Delta_{\text{RHS}} - \Delta_{\text{LHS}}$
		\\
		\hline
		2.0 & 99.43 & 0.0006728 & 0.006708\\
		3.0 & 99.41 & 0.0004067 & 0.005914\\
		4.0 & 99.37 & 0.0003746 & 0.006757\\
		5.0 & 99.34 & 0.0003430 & 0.005860\\
		\hline
	\end{tabular}
\end{table}
The results in Table \ref{table: theoretical upper bound tightness verification} show the tightness of our upper bound in Theorem \ref{theorem: surrogate function}. It shows that the differences between $\Delta_{\text{RHS}}$ and $\Delta_{\text{LHS}}$ under various $\lambda$'s are very small.

\vspace{-0.2cm}
\subsection{Sensitivity of regularization hyperparameter $\lambda$}
\vspace{-0.2cm}

The regularization parameter $\lambda$ is an important hyperparameter in our proposed method. We show how the regularization parameter affects the performance of our robust classifiers by numerical experiments on two datasets, MNIST and CIFAR10. For both datasets, we minimize the loss in Eqn. \eqref{equ: multi-class loss} to learn robust classifiers for multi-class problems, where we choose $\cL$ as the cross-entropy loss. 

\medskip
\noindent{\textbf{MNIST setup.}} We use the CNN which has two convolutional layers, followed by two fully-connected layers. The output size of the last layer is 10. We set perturbation $\epsilon = 0.1$, perturbation step size $\eta_1 = 0.01$, number of iterations $K = 20$, learning rate $\eta_2 = 0.01$, batch size $m = 128$, and run $50$ epochs on the training dataset. To evaluate the robust error, we apply FGSM$^k$ (white-box) attack with $40$ iterations and $0.005$ step size. The results are in Table \ref{table: sensitivity analysis}.

\medskip
\noindent{\textbf{CIFAR10 setup.}} We apply ResNet-18~\cite{he2016deep} for classification. The output size of the last layer is 10. We set perturbation $\epsilon = 0.031$, perturbation step size $\eta_1 = 0.007$, number of iterations $K = 10$, learning rate $\eta_2 = 0.1$, batch size $m = 128$, and run $100$ epochs on the training dataset. To evaluate the robust error, we apply FGSM$^k$ (white-box) attack with $20$ iterations and the step size is $0.003$. The results are in Table \ref{table: sensitivity analysis}.

\begin{table*}
	\caption{Sensitivity of regularization hyperparameter $\lambda$ on MNIST and CIFAR10 datasets.}
	\label{table: sensitivity analysis}
	\centering
	\begin{tabular}{c|cc||cc}%
		\hline
		& \multicolumn{2}{c}{MNIST} & \multicolumn{2}{c}{CIFAR10}\\
		\hline
		$1/\lambda$ & $\cA_\adv(f)$ $(\%)$   & $\cA_\nat(f)$  $(\%)$ & $\cA_\adv(f)$ $(\%)$ & $\cA_\nat(f)$  $(\%)$\\
		\hline
		0.1 & 91.09 $\pm$ 0.0385 & 99.41 $\pm$ 0.0235 & 26.53 $\pm$ 1.1698 & 91.31 $\pm$ 0.0579 \\
		0.2 & 92.18 $\pm$ 0.0450 & 99.38 $\pm$ 0.0094 & 37.71 $\pm$ 0.6743 & 89.56 $\pm$ 0.2154\\
		0.4 & 93.21 $\pm$ 0.0660 & 99.35 $\pm$ 0.0082 & 41.50 $\pm$ 0.3376 & 87.91 $\pm$ 0.2944\\
		0.6 & 93.87 $\pm$ 0.0464 & 99.33 $\pm$ 0.0141 & 43.37 $\pm$ 0.2706 & 87.50 $\pm$ 0.1621\\
		0.8 & 94.32 $\pm$ 0.0492 & 99.31 $\pm$ 0.0205 & 44.17 $\pm$ 0.2834 & 87.11 $\pm$ 0.2123\\
		1.0 & 94.75 $\pm$ 0.0712 & 99.28 $\pm$ 0.0125 & 44.68 $\pm$ 0.3088 & 87.01 $\pm$ 0.2819\\
		2.0 & 95.45 $\pm$ 0.0883 & 99.29 $\pm$ 0.0262 & 48.22 $\pm$ 0.0740 & 85.22 $\pm$ 0.0543\\
		3.0 & 95.57 $\pm$ 0.0262 & 99.24 $\pm$ 0.0216 & 49.67 $\pm$ 0.3179 & 83.82 $\pm$ 0.4050\\
		4.0 & 95.65 $\pm$ 0.0340 & 99.16 $\pm$ 0.0205 & 50.25 $\pm$ 0.1883 & 82.90 $\pm$ 0.2217\\
		5.0 & 95.65 $\pm$ 0.1851 & 99.16 $\pm$ 0.0403 & 50.64 $\pm$ 0.3336 & 81.72 $\pm$ 0.0286\\
		\hline
	\end{tabular}
\end{table*}

\comment{
\begin{table}
	\caption{Sensitivity analysis on regularization parameter $\lambda$ on CIFAR10 dataset.}
	\label{table: sensitivity analysis - CIFAR10}
	\centering
	\begin{tabular}{c|cc}%
		\hline
		$1/\lambda$ & $\cR_\adv(f)$ $(\%)$ & $\cR_\nat(f)$ $(\%)$
		\\
		\hline
		0.1 & 26.53 $\pm$ 1.3686 & 91.31 $\pm$ 0.0033 \\
		0.2 & 37.71 $\pm$ 0.4547 & 89.56 $\pm$ 0.0464 \\
		0.4 & 41.50 $\pm$ 0.1140 & 87.91 $\pm$ 0.0866\\
		0.6 & 43.37 $\pm$ 0.0732 & 87.50 $\pm$ 0.0262\\
		0.8 & 44.17 $\pm$ 0.0802 & 87.11 $\pm$ 0.0450\\
		1.0 & 44.68 $\pm$ 0.0953 & 87.01 $\pm$ 0.0794\\
		2.0 & 48.22 $\pm$ 0.0054 & 85.22 $\pm$ 0.0029\\
		3.0 & 49.67 $\pm$ 0.1010 & 83.82 $\pm$ 0.1640\\
		4.0 & 50.25 $\pm$ 0.0354 & 82.90 $\pm$ 0.0491\\
		5.0 & 50.64 $\pm$ 0.1112 & 81.72 $\pm$ 0.0008\\
		\hline
	\end{tabular}
\end{table}
}

We observe that as the regularization parameter $1/\lambda$ increases, the natural accuracy $\cA_\nat(f)$ decreases while the robust accuracy $\cA_\adv(f)$ increases, which verifies our theory on the trade-off between robustness and accuracy. Note that for MNIST dataset, the natural accuracy does not decrease too much as the regularization term $1/\lambda$ increases, which is different from the results of CIFAR10. This is probably because the classification task for MNIST is easier.  Meanwhile, our proposed method is not very sensitive to the choice of $\lambda$. Empirically, when we set the hyperparameter $1/\lambda$ in $[1, 10]$,  our method is able to learn classifiers with both high robustness and high accuracy.  We will set $1/\lambda$ as either 1 or 6 in the following experiments.

\vspace{-0.3cm}
\subsection{Adversarial defenses under various attacks}
\vspace{-0.2cm}
Previously, \cite{athalye2018obfuscated} showed that 7 defenses in ICLR 2018 which relied on obfuscated
gradients may easily break down. In this section, we verify the effectiveness of our method with the same experimental setup under both white-box and black-box threat models.

\medskip
\noindent{\textbf{MNIST setup.}} We use the CNN architecture in~\cite{carlini2017towards} with four convolutional layers, followed by three fully-connected layers. We set perturbation $\epsilon = 0.3$, perturbation step size $\eta_1 = 0.01$, number of iterations $K = 40$, learning rate $\eta_2 = 0.01$, batch size $m = 128$, and run $100$ epochs on the training dataset. 

\medskip
\noindent{\textbf{CIFAR10 setup.}} We use the same neural network architecture as~\cite{madry2018towards}, i.e., the wide residual network WRN-34-10~\cite{zagoruyko2016wide}. We set perturbation $\epsilon = 0.031$, perturbation step size $\eta_1 = 0.007$, number of iterations $K = 10$, learning rate $\eta_2 = 0.1$, batch size $m = 128$, and run $100$ epochs on the training dataset. 

\vspace{-0.2cm}
\subsubsection{White-box attacks}\label{subsec:white-box}
\vspace{-0.2cm}
We summarize our results in Table \ref{table: icml best paper defense} together with the results from~\cite{athalye2018obfuscated}. We also implement methods in~\cite{zheng2016improving,kurakin2016adversarial,ross2017improving} on the CIFAR10 dataset as they are also regularization based methods. For MNIST dataset, we apply FGSM$^k$ (white-box) attack with $40$ iterations and the step size is $0.01$. For CIFAR10 dataset, we apply FGSM$^k$ (white-box) attack with $20$ iterations and the step size is $0.003$, under which the defense model in~\cite{madry2018towards} achieves $47.04\%$ robust accuracy. Table \ref{table: icml best paper defense} shows that our proposed defense method can significantly improve the robust accuracy of models, which is able to achieve robust accuracy as high as $56.61\%$. We also evaluate our robust model on MNIST dataset under the same threat model as in~\cite{samangouei2018defense} (C\&W white-box attack~\cite{carlini2017towards}), and the robust accuracy is $99.46\%$. See appendix for detailed information of models in Table~\ref{table: icml best paper defense}.

\begin{table*}
	\caption{Comparisons of TRADES with prior defense models under white-box attacks.}
	\label{table: icml best paper defense}
	\centering
	\begin{tabular}{c||c|c|c|c|c|c}%
		\hline
		Defense & Defense type & Under which attack & Dataset & Distance & $\cA_\nat(f)$   & $\cA_\adv(f)$  
		\\
		\hline
		\hline
		\cite{buckman2018thermometer} & gradient mask & \cite{athalye2018obfuscated} & CIFAR10 & $0.031$ ($\ell_\infty$) & - & 0\% \\
		\cite{ma2018characterizing} & gradient mask & \cite{athalye2018obfuscated} & CIFAR10 &  $0.031$ ($\ell_\infty$) & - & 5\% \\
		\cite{dhillon2018stochastic} & gradient mask & \cite{athalye2018obfuscated} & CIFAR10 &  $0.031$ ($\ell_\infty$) & - & 0\% \\
		\cite{song2018pixeldefend} & gradient mask & \cite{athalye2018obfuscated} & CIFAR10 & $0.031$ ($\ell_\infty$) & - & 9\% \\ 
		\cite{na2017cascade} & gradient mask & \cite{athalye2018obfuscated} & CIFAR10 & $0.015$ ($\ell_\infty$) & - & 15\% \\ 
		\cite{wong1805scaling} & robust opt. & FGSM$^{20}$ (PGD) & CIFAR10 & $0.031$ ($\ell_\infty$) & 27.07\% & 23.54\% \\
		\cite{madry2018towards} & robust opt. & FGSM$^{20}$ (PGD) & CIFAR10 & $0.031$ ($\ell_\infty$) & 87.30\% & \textbf{47.04\%} \\
		\cite{zheng2016improving} & regularization & FGSM$^{20}$ (PGD) & CIFAR10 & $0.031$ ($\ell_\infty$) & 94.64\% & 0.15\% \\
		\cite{kurakin2016adversarial} & regularization & FGSM$^{20}$ (PGD) & CIFAR10 & $0.031$ ($\ell_\infty$) & 85.25\% & 45.89\% \\
		\cite{ross2017improving} & regularization & FGSM$^{20}$ (PGD) & CIFAR10 & $0.031$ ($\ell_\infty$) & 95.34\% & 0\% \\
		{TRADES} ($1/\lambda=1$) & regularization & FGSM$^{1,000}$ (PGD)  & CIFAR10 &  $0.031$ ($\ell_\infty$) & 88.64\% & 48.90\% \\
		{TRADES} ($1/\lambda=6$) & regularization & FGSM$^{1,000}$ (PGD) & CIFAR10 &  $0.031$ ($\ell_\infty$) & 84.92\% & \textbf{56.43\%} \\
		{TRADES} ($1/\lambda=1$) & regularization & FGSM$^{20}$ (PGD)  & CIFAR10 &  $0.031$ ($\ell_\infty$) & 88.64\% & 49.14\% \\
		{TRADES} ($1/\lambda=6$) & regularization & FGSM$^{20}$ (PGD) & CIFAR10 &  $0.031$ ($\ell_\infty$) & 84.92\% & \textbf{56.61\%} \\
		{TRADES} ($1/\lambda=1$) & regularization & DeepFool ($\ell_\infty$) & CIFAR10 &  $0.031$ ($\ell_\infty$) & 88.64\% & 59.10\% \\
		{TRADES} ($1/\lambda=6$) & regularization & DeepFool ($\ell_\infty$) & CIFAR10 &  $0.031$ ($\ell_\infty$) & 84.92\% & 61.38\% \\
		{TRADES} ($1/\lambda=1$) & regularization & LBFGSAttack  & CIFAR10 &  $0.031$ ($\ell_\infty$) & 88.64\% & 84.41\% \\
		{TRADES} ($1/\lambda=6$) & regularization & LBFGSAttack  & CIFAR10 &  $0.031$ ($\ell_\infty$) & 84.92\% & 81.58\% \\
		{TRADES} ($1/\lambda=1$) & regularization & MI-FGSM  & CIFAR10 &  $0.031$ ($\ell_\infty$) & 88.64\% & 51.26\% \\
		{TRADES} ($1/\lambda=6$) & regularization & MI-FGSM & CIFAR10 &  $0.031$ ($\ell_\infty$) & 84.92\% & 57.95\% \\
		{TRADES} ($1/\lambda=1$) & regularization & C\&W & CIFAR10 &  $0.031$ ($\ell_\infty$) & 88.64\% & 84.03\% \\
		{TRADES} ($1/\lambda=6$) & regularization & C\&W & CIFAR10 &  $0.031$ ($\ell_\infty$) & 84.92\% & 81.24\% \\
		\hline
		\cite{samangouei2018defense} & gradient mask & \cite{athalye2018obfuscated} & MNIST & $0.005$ ($\ell_2$) & - & 55\% \\
		\cite{madry2018towards} & robust opt. & FGSM$^{40}$ (PGD) & MNIST & $0.3$ ($\ell_\infty$) & 99.36\% & 96.01\% \\
		{TRADES} ($1/\lambda=6$) & regularization & FGSM$^{1,000}$ (PGD) & MNIST &  $0.3$ ($\ell_\infty$) & 99.48\% & 95.60\% \\
		{TRADES} ($1/\lambda=6$) & regularization & FGSM$^{40}$ (PGD) & MNIST &  $0.3$ ($\ell_\infty$) & 99.48\% & 96.07\% \\
		{TRADES} ($1/\lambda=6$) & regularization & C\&W & MNIST &  $0.005$ ($\ell_2$) & 99.48\% & 99.46\% \\
		\hline
	\end{tabular}
\end{table*}

\vspace{-0.2cm}
\subsubsection{Black-box attacks}\label{subsec:black-box}

We verify the robustness of our models under black-box attacks. We first train models without using adversarial training on the MNIST and CIFAR10 datasets. We use the same network architectures that are specified in the beginning of this section, i.e., the CNN architecture in~\cite{carlini2017towards} and the WRN-34-10 architecture in~\cite{zagoruyko2016wide}. We denote these models by naturally trained models (\emph{Natural}). The accuracy of the naturally trained CNN model is $99.50\%$ on the MNIST dataset. The accuracy of the naturally trained WRN-34-10 model is $95.29\%$ on the CIFAR10 dataset. We also implement the method proposed in~\cite{madry2018towards} on both datasets. We denote these models by Madry's models (\emph{Madry}). The accuracy of~\cite{madry2018towards}'s CNN model is $99.36\%$ on the MNIST dataset. The accuracy of~\cite{madry2018towards}'s WRN-34-10 model is $85.49\%$ on the CIFAR10 dataset.

For both datasets, we use FGSM$^{k}$ (black-box) method to attack various defense models. For MNIST dataset, we set perturbation $\epsilon=0.3$ and apply FGSM$^k$ (black-box) attack with $40$ iterations and the step size is $0.01$. For CIFAR10 dataset, we set $\epsilon=0.031$ and apply FGSM$^k$ (black-box) attack with $20$ iterations and the step size is $0.003$. Note that the setup is the same as the setup specified in Section~\ref{subsec:white-box}. We summarize our results in Table \ref{table: MNIST FGSM-40 black-box defense} and Table \ref{table: CIFAR10 FGSM-20 black-box defense}. In both tables, we use two source models (noted in the parentheses) to generate adversarial perturbations: we compute the perturbation directions according to the gradients of the source models on the input images. It shows that our models are more robust against black-box attacks transfered from naturally trained models and \cite{madry2018towards}'s models. Moreover, our models can generate stronger adversarial examples for black-box attacks compared with naturally trained models and \cite{madry2018towards}'s models.

\begin{table}
	\caption{Comparisons of TRADES with prior defenses under black-box FGSM$^{40}$ attack on the MNIST dataset. The models inside parentheses are source models which provide gradients to adversarial attackers. We provide the average cross-entropy loss value $\cL(f(\X),\Y)$ of each defense model in the bracket. The defense model `Madry' is the same model as in the antepenultimate line of Table~\ref{table: icml best paper defense}. The defense model `TRADES' is the same model as in the penultimate line of Table~\ref{table: icml best paper defense}.}
	\vspace{+0.2cm}
	\label{table: MNIST FGSM-40 black-box defense}
	\centering
	\begin{tabular}{c||cc}%
		\hline
		Defense Model & \multicolumn{2}{c}{Robust Accuracy $\cA_\adv(f)$}
		\\
		\hline
		{Madry} & 97.43\% [0.0078484]  & (Natural)   
		\\ 
		\hline
		{TRADES} & \textbf{97.63\%} [0.0075324] &  (Natural)
		\\
		\hline
		{Madry} & 97.38\%  [0.0084962] & (Ours) \\ 
		\hline
		{TRADES} &  \textbf{97.66\%}  [0.0073532] &  (Madry)  
		\\
		\hline
	\end{tabular}
\end{table}

\vspace{-0.3cm}
\begin{table}
	\caption{Comparisons of TRADES with prior defenses under black-box FGSM$^{20}$ attack on the CIFAR10 dataset. The models inside parentheses are source models which provide gradients to adversarial attackers.  We provide the average cross-entropy loss value of each defense model in the bracket. The defense model `Madry' is implemented based on \cite{madry2018towards}, and  the defense model `TRADES' is the same model as in the 11th line of Table~\ref{table: icml best paper defense}.}
	\vspace{+0.2cm}
	\label{table: CIFAR10 FGSM-20 black-box defense}
	\centering
	\begin{tabular}{c||cc}%
		\hline
		Defense Model & \multicolumn{2}{c}{Robust Accuracy $\cA_\adv(f)$}
		\\
		\hline
		{Madry} & 84.39\% [0.0519784] & (Natural) \\ 
		\hline
		{TRADES} & \textbf{87.60\%} [0.0380258] & (Natural)  \\
		\hline
		{Madry} & 66.00\% [0.1252672]  &(Ours) \\ 
		\hline
		{TRADES} &  \textbf{70.14}\% [0.0885364] & (Madry)  \\
		\hline
	\end{tabular}
\end{table}

\subsection{Case study: NeurIPS 2018 Adversarial Vision Challenge}

\medskip
\noindent{\textbf{Competition settings.}}
In the adversarial competition, the adversarial attacks and defenses are under the black-box setting. The dataset in this competition is Tiny ImageNet, which consists of 550,000 data (with our data augmentation) and 200 classes. The robust models only return label predictions instead of explicit gradients and confidence scores.
The task for robust models is to defend against adversarial examples that are generated by the top-5 submissions in the un-targeted attack track. The score for each defense model is evaluated by the smallest perturbation distance that makes the defense model fail to output correct labels.

\medskip
\noindent{\textbf{Competition results.}}
The methodology in this paper was applied to the competition, where our entry ranked the 1st place. We implemented our method to train ResNet models.
We report the mean $\ell_2$ perturbation distance of the top-6 entries in Figure \ref{figure: NIPS competition}. It shows that our method outperforms other approaches with a large margin. In particular, we surpass the runner-up submission by $11.41\%$ in terms of mean $\ell_2$ perturbation distance.

\begin{figure}
\centering
\includegraphics[scale=0.8]{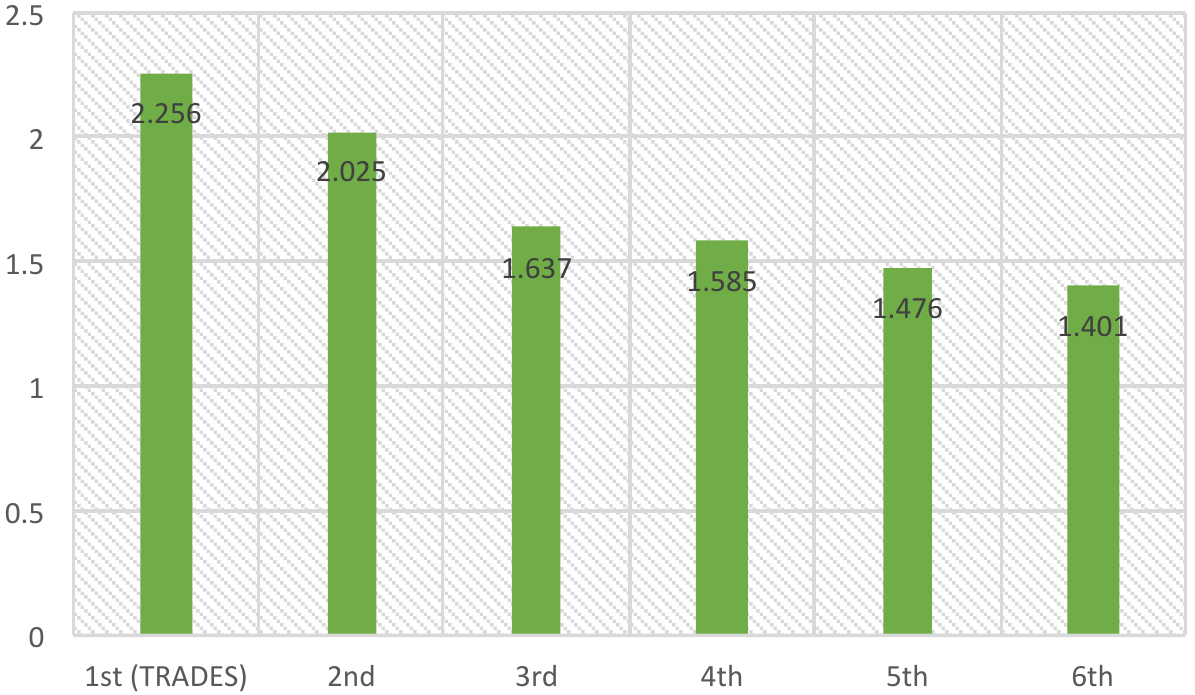}
\caption{Top-6 results (out of \textasciitilde2,000 submissions) in the NeurIPS 2018 Adversarial Vision Challenge. The vertical axis represents the mean $\ell_2$ perturbation distance that makes robust models fail to output correct labels.}
\label{figure: NIPS competition}
\end{figure}

\section{Conclusions}

In this paper, we study the problem of adversarial defenses against structural perturbations around input data. We focus on the trade-off between robustness and accuracy, and show an upper bound on the gap between robust error and optimal natural error. Our
result advances the state-of-the-art work and matches the lower bound in the worst-case scenario. The bounds motivate us to minimize a new form of regularized surrogate loss, TRADES, for adversarial training. Experiments on real datasets and adversarial competition demonstrate the effectiveness of our proposed algorithms.
It would be interesting to combine our methods with other related line of research on adversarial defenses, e.g., feature denoising technique~\cite{xie2018feature} and network architecture design~\cite{cisse2017parseval}, to achieve more robust learning systems.

\medskip
\noindent\textbf{Acknowledgements.} We thank Maria-Florina Balcan, Avrim Blum, Zico Kolter, and Aleksander Mądry for valuable comments and discussions.

\bibliographystyle{alpha}
\bibliography{reference}

\newpage
\onecolumn
\appendix

\section{Other Related Works}

\medskip
\noindent{\textbf{Attack methods.}} Although deep neural networks have achieved great progress in various areas~\cite{zhang2019deep,zhang2018stackelberg}, they are brittle to adversarial attacks. Adversarial attacks have been extensively studied in the recent
years. One of the baseline attacks to deep nerual networks is the \emph{Fast Gradient Sign Method} (FGSM)~\cite{goodfellow6572explaining}. FGSM computes an adversarial example as
\begin{equation*}
\x':=\x+\epsilon\sign(\nabla_{\x} \phi(f(\x)y)),
\end{equation*}
where $\x$ is the input instance, $y$ is the label, $f:\cX\rightarrow \R$ is the \emph{score function} (parametrized by deep nerual network for example) which maps an instance to its confidence value of being positive, and $\phi(\cdot)$ is a surrogate of 0-1 loss. A more powerful yet natural extension of FGSM is the multi-step variant FGSM$^k$ (also known as PGD attack)~\cite{kurakin2016adversarial}. FGSM$^k$ applies \emph{projected gradient descent} by $k$ times:
\begin{equation*}
\x_{t+1}':=\Pi_{\bbB(\x,\epsilon)}(\x_t'+\epsilon\sign(\nabla_{\x} \phi(f(\x_t')y))),
\end{equation*}
where $\x_t'$ is the $t$-th iteration of the algorithm with $\x_0':=\x$ and $\Pi_{\bbB(\x,\epsilon)}$ is the projection operator onto the ball $\bbB(\x,\epsilon)$.
Both FGSM and FGSM$^k$ are approximately solving (the linear approximation of) maximization problem:
\begin{equation*}
\max_{\x'\in\bbB(\x,\epsilon)} \phi(f(\x')y).
\end{equation*}
They can be adapted to the purpose of black-box attacks by running the algorithms on another similar network which is white-box to the algorithms~\cite{tramer2017ensemble}. Though defenses that cause obfuscated
gradients defeat iterative optimization based attacks, \cite{athalye2018obfuscated} showed that defenses relying on this
effect can be circumvented. Other attack methods include MI-FGSM~\cite{dong2018boosting} and LBFGS attacks~\cite{tabacof2016exploring}.

\medskip
\noindent{\textbf{Robust optimization based defenses.}} Compared with attack methods, adversarial defense methods are relatively fewer. Robust optimization based defenses are inspired by the above-mentioned attacks. Intuitively, the methods train a network by fitting its parameters to the adversarial examples:
\begin{equation}
\label{equ: old loss app}
\min_{f} \bbE \left\{ \max_{\X'\in\bbB(\X,\epsilon)}\phi(f(\X')Y)\right\}.
\end{equation}
Following this framework, \cite{huang2015learning,shaham2015understanding}  considered one-step adversaries, while \cite{madry2018towards} worked with multi-step methods for the inner maximization problem.
There are, however, two critical differences between the robust optimization based defenses
and the present paper. Firstly, robust optimization based defenses lack of theoretical guarantees. Secondly, such methods do not consider the trade-off between accuracy and robustness.

\medskip
\noindent{\textbf{Relaxation based defenses.}} We mention another related line of research in adversarial defenses---relaxation based defenses. Given that the inner maximization in problem \eqref{equ: old loss app} might be hard to solve due to the non-convexity nature of deep neural networks, \cite{kolter2017provable} and \cite{raghunathan2018certified} considered a convex outer approximation of the set of activations reachable through a norm-bounded perturbation for one-hidden-layer neural networks. \cite{wong1805scaling} later scaled the methods to larger models, and \cite{raghunathan2018semidefinite} proposed a tighter convex approximation. \cite{sinha2017certifiable,NIPS2018_7779} considered a Lagrangian penalty formulation of perturbing the underlying data distribution in a Wasserstein ball. These approaches, however, do not apply when the activation function is ReLU.

\medskip
\noindent{\textbf{Theoretical progress.}} Despite a large amount of empirical works on adversarial defenses, many fundamental questions remain open in theory. There are a few preliminary explorations in recent years. \cite{fawzi2018adversarial} derived upper bounds on the robustness to perturbations of any classification function, under the assumption that the data is
generated with a smooth generative model. From computational aspects, \cite{bubeck2018adversarial1,bubeck2018adversarial2} showed that adversarial examples in machine learning are likely not due to information-theoretic limitations, but rather it could be due to computational hardness. From statistical aspects, \cite{schmidt2018adversarially} showed that the sample complexity of robust training can be significantly larger than that of standard training. This gap holds irrespective of the training algorithm or the model family. \cite{NIPS2018_7307} and \cite{yin2018rademacher} studied the uniform convergence of robust error $\cR_\adv(f)$ by extending the classic VC and Rademacher arguments to the case of adversarial learning, respectively. A recent work demonstrates the existence of trade-off between accuracy and robustness~\cite{tsipras2018robustness}, without providing a practical algorithm to address it.

\section{Proofs of Main Results}

In this section, we provide the proofs of our main results.

\subsection{Proof of Theorem \ref{theorem: surrogate function}}

\medskip
\noindent\textbf{Theorem~\ref{theorem: surrogate function} (restated).}
\emph{Let $\cR_\phi(f):=\bbE\phi(f(\X)Y)$ and $\cR_\phi^*:=\min_f \cR_\phi(f)$. Under Assumption \ref{assumption: classification-calibrated}, for any non-negative loss function $\phi$ such that $\phi(0)\ge 1$, any measurable $f:\cX\rightarrow \R$, any probability distribution on $\cX\times\{\pm 1\}$, and any $\lambda>0$, we have
\begin{equation*}
\begin{split}
\quad\cR_\adv(f)-\cR_\nat^*&\le \psi^{-1}(\cR_\phi(f)-\cR_\phi^*)+\Pr[\X\in\bbB(\boundary(f),\epsilon),f(\X)Y>0]\\
&\le \psi^{-1}(\cR_\phi(f)-\cR_\phi^*)+\bbE \max_{\X'\in\bbB(\X,\epsilon)}\phi(f(\X')f(\X)/\lambda).
\end{split}
\end{equation*}
}

\begin{proof}
By Eqn. \eqref{eqn.fundamentalequation}, $\cR_{\adv}(f)-\cR_{\nat}^*=\cR_\nat(f)-\cR_{\nat}^*+\cR_{\text{bdy}}(f)\le \psi^{-1}(\cR_\phi(f)-\cR_\phi^*)+\cR_{\text{bdy}}(f)$, where the last inequality holds because we choose $\phi$ as a classification-calibrated loss~\cite{bartlett2006convexity}. This leads to the first inequality.

Also, notice that
\begin{equation*}
\begin{split}
\Pr[\X\in\bbB(\boundary(f),\epsilon),f(\X)Y>0]&\le \Pr[\X\in\bbB(\boundary(f),\epsilon)]\\
&= \bbE \max_{\X'\in\bbB(\X,\epsilon)}\1\{f(\X')\not=f(\X)\}\\
&= \bbE \max_{\X'\in\bbB(\X,\epsilon)}\1\{f(\X')f(\X)/\lambda<0\}\\
&\le \bbE \max_{\X'\in\bbB(\X,\epsilon)}\phi(f(\X')f(\X)/\lambda).
\end{split}
\end{equation*}
This leads to the second inequality.
\end{proof}

\subsection{Proof of Theorem \ref{theorem: tightness of surrogate loss}}

\medskip
\noindent\textbf{Theorem~\ref{theorem: tightness of surrogate loss} (restated).}
\emph{Suppose that $|\cX|\ge 2$. Under Assumption \ref{assumption: classification-calibrated}, for any non-negative loss function $\phi$ such that $\phi(x)\rightarrow 0$ as $x\rightarrow +\infty$, any $\xi>0$, and any $\theta\in[0,1]$, there exists a probability distribution on $\cX\times \{\pm 1\}$, a function $f:\R^d\rightarrow\R$, and a regularization parameter $\lambda>0$ such that
$\cR_\adv(f)-\cR_\nat^*=\theta$
and
\begin{equation*}
\psi\Big(\theta-\bbE \max_{\X'\in\bbB(\X,\epsilon)}\phi(f(\X')f(\X)/\lambda)\Big)\le \cR_\phi(f)-\cR_\phi^*\le \psi\left(\theta-\bbE \max_{\X'\in\bbB(\X,\epsilon)}\phi(f(\X')f(\X)/\lambda)\right)+\xi.
\end{equation*}}

\begin{proof}
The first inequality follows from Theorem \ref{theorem: surrogate function}. Thus it suffices to prove the second inequality.

Fix $\epsilon>0$ and $\theta\in[0,1]$. By the definition of $\psi$ and its continuity, we can choose $\gamma,\alpha_1,\alpha_2\in[0,1]$ such that $\theta=\gamma\alpha_1+(1-\gamma)\alpha_2$ and $\psi(\theta)\ge \gamma\tilde{\psi}(\alpha_1)+(1-\gamma)\tilde{\psi}(\alpha_2)-\epsilon/3$. For two distinct points $\x_1,\x_2\in \cX$, we set $\cP_\cX$ such that $\Pr[\X=\x_1]=\gamma$, $\Pr[\X=\x_2]=1-\gamma$, $\eta(\x_1)=(1+\alpha_1)/2$, and $\eta(\x_2)=(1+\alpha_2)/2$. By the definition of $H^-$, we choose function $f:\R^d\rightarrow\R$ such that $f(\x)< 0$ for all $\x\in\cX$, $C_{\eta(\x_1)}(f(\x_1))\le H^-(\eta(\x_1))+\epsilon/3$, and $C_{\eta(\x_2)}(f(\x_2))\le H^-(\eta(\x_2))+\epsilon/3$. By the continuity of $\psi$, there is an $\epsilon'>0$ such that $\psi(\theta)\le \psi(\theta-\epsilon_0)+\epsilon/3$ for all $0\le \epsilon_0<\epsilon'$. We also note that there exists an $\lambda_0>0$ such that for any $0<\lambda<\lambda_0$, we have
\begin{equation*}
0\le \bbE \max_{\X'\in\bbB(\X,\epsilon)}\phi(f(\X')f(\X)/\lambda)<\epsilon'.
\end{equation*}
Thus, we have
\begin{equation*}
\begin{split}
\cR_\phi(f)-\cR_\phi^*&=\bbE\phi(Yf(\X))-\inf_f \bbE\phi(Yf(\X))\\
&=\gamma [C_{\eta(\x_1)}(f(\x_1))-H(\eta(\x_1))]+(1-\gamma)[C_{\eta(\x_2)}(f(\x_2))-H(\eta(\x_2))]\\
&\le \gamma [H^-(\eta(\x_1))-H(\eta(\x_1))]+(1-\gamma)[H^-(\eta(\x_2))-H(\eta(\x_2))]+\epsilon/3\\
&=\gamma \tilde\psi(\alpha_1)+(1-\gamma)\tilde{\psi}(\alpha_2)+\epsilon/3\\
&\le \psi(\theta)+2\epsilon/3\\
&\le \psi\left(\theta-\bbE \max_{\X'\in\bbB(\X,\epsilon)}\phi(f(\X')f(\X)/\lambda)\right)+\epsilon.
\end{split}
\end{equation*}
Furthermore, by Lemma \ref{lemma: risk equality},
\begin{equation*}
\begin{split}
\cR_\adv(f)-\cR_\nat^*&=\bbE[\1\{\sign(f(\X))\not=\sign(f^*(\X)),\X\in\bbB(\boundary(f),\epsilon)^\perp\}|2\eta(\X)-1|]\\
&\quad +\Pr[\X\in \bbB(\boundary(f),\epsilon),\sign(f^*(\X))=Y]\\
&=\bbE|2\eta(\X)-1|\\
&=\gamma (2\eta(\x_1)-1)+(1-\gamma)(2\eta(\x_2)-1)\\
&=\theta,
\end{split}
\end{equation*}
where $f^*$ is the Bayes optimal classifier which outputs ``positive'' for all data points.
\end{proof}

\section{Extra Theoretical Results}
In this section, we provide extra theoretical results for adversarial defenses.

\subsection{Adversarial vulnerability under log-concave distributions}

Theorem \ref{theorem: surrogate function} states that for any classifier $f$, the value
$\Pr[\X\in\bbB(\boundary(f),\epsilon)]$
characterizes the robustness of the classifier.
In this section, we show that among all classifiers such that $\Pr[\sign(f(\X))=+1]=1/2$, linear classifier minimizes
\begin{equation}
\label{equ: boundary measure}
\liminf_{\epsilon\rightarrow+0}\frac{\Pr[\X\in\bbB(\boundary(f),\epsilon)]}{\epsilon},
\end{equation}
provided that the marginal distribution over $\cX$ is products of log-concave measures. A measure is \emph{log-concave} if the logarithm of its density is a concave function. The class of log-concave measures contains many well-known (classes of) distributions as special cases, such as Gaussian and uniform measure over ball.

Our results are inspired by the isoperimetric inequality of log-concave distributions by the work of \cite{barthe2001extremal}. Intuitively, the isoperimetric problem consists in finding subsets of prescribed measure, such that its measure increases the less under enlargement. Our analysis leads to the following guarantee on the quantity \eqref{equ: boundary measure}.

\begin{theorem}
\label{theorem: vulnerability}
Let $\mu$ be an absolutely continuous log-concave probability measure on $\R$ with even density function and let $\mu^{\otimes d}$ be the products of $\mu$ with dimension $d$. Denote by $d\mu=e^{-M(x)}$, where $M:\R\rightarrow[0,\infty]$ is convex. Assume that $M(0)=0$. If $\sqrt{M(x)}$ is a convex function, then for every integer $d$ and any classifier $f$ with $\Pr[\sign(f(\X))=+1]=1/2$, we have
\begin{equation*}
\liminf_{\epsilon\rightarrow+0}\frac{\Pr_{\X\sim \mu^{\otimes d}}[\X\in\bbB(\boundary(f),\epsilon)]}{\epsilon}\ge c
\end{equation*}
for an absolute constant $c>0$. Furthermore, among all such probability measures and classifiers, the linear classifier over products of Gaussian measure with mean $0$ and variance $1/(2\pi)$ achieves the lower bound.
\end{theorem}

Theorem \ref{theorem: vulnerability} claims that under the products of log-concave distributions, the quantity $\Pr[\X\in\bbB(\boundary(f),\epsilon)]$ increases with rate at least $\Omega(\epsilon)$ for all classifier $f$, among which the linear classifier achieves the minimal value.

\begin{figure}[t]
\centering
\subfigure{
\includegraphics[width=6.5cm]{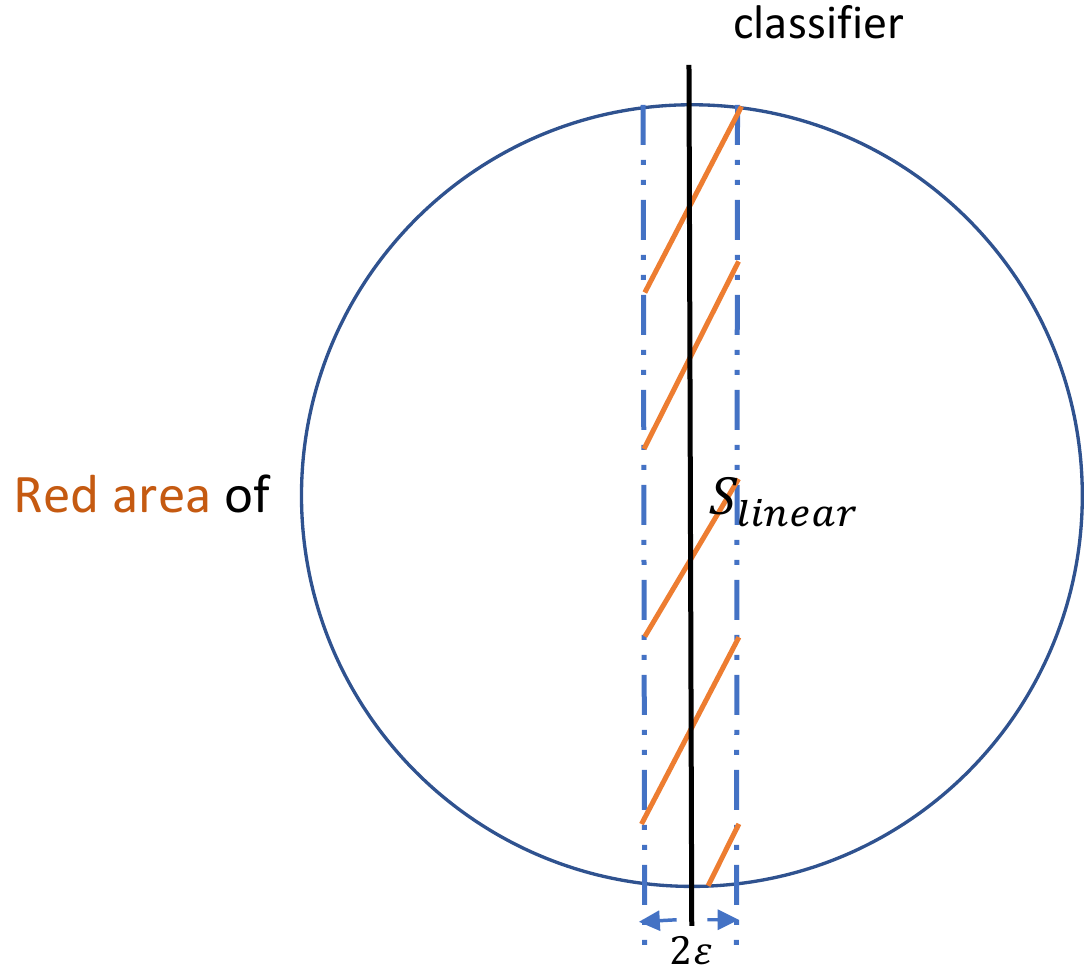}}
\subfigure{
\includegraphics[width=6.8cm]{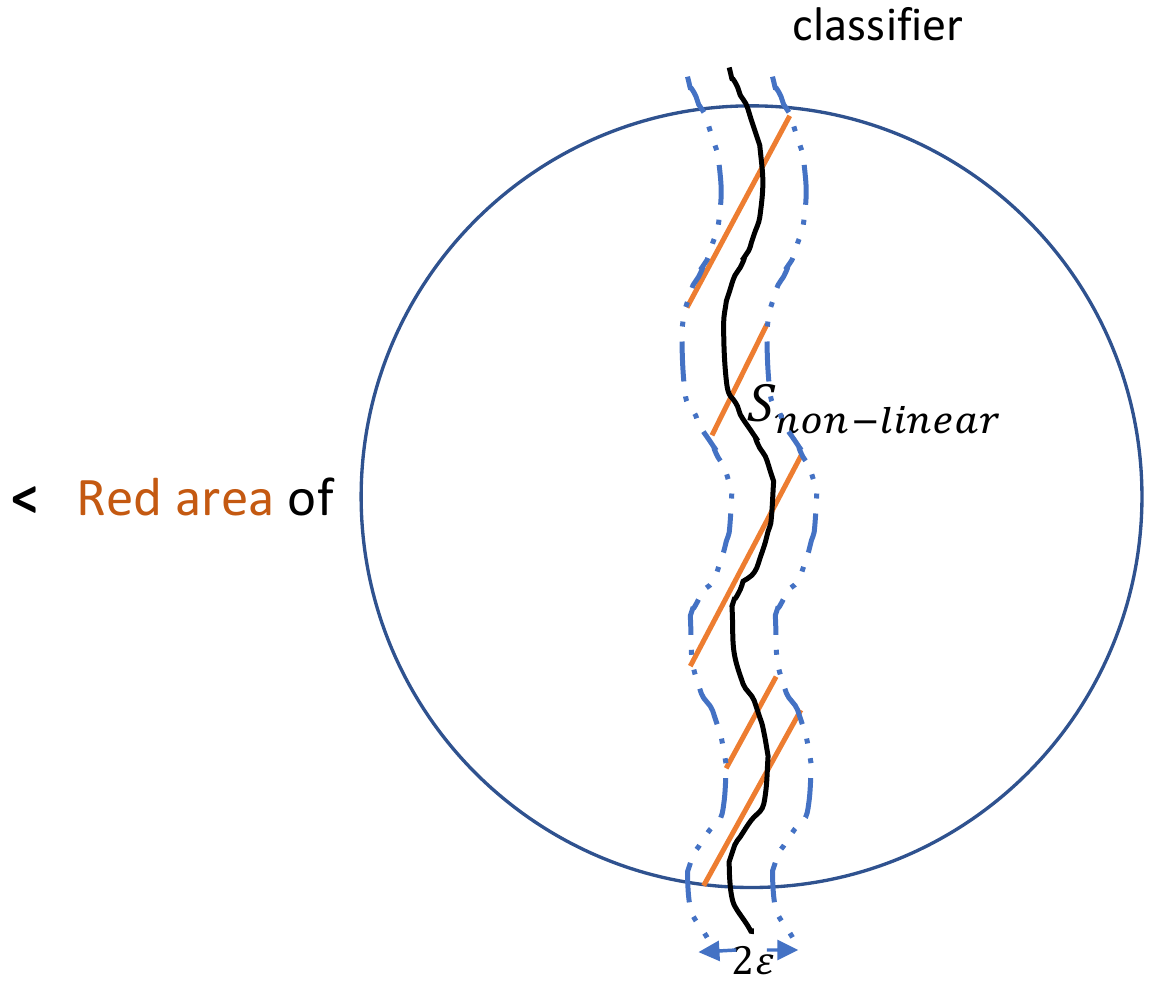}}
\caption{\textbf{Left figure:} boundary neighborhood of linear classifier. \textbf{Right figure:} boundary neighborhood of non-linear classifier. Theorem \ref{theorem: vulnerability} shows that the mass of $S_{\text{linear}}$ is smaller than the mass of $S_{\text{non-linear}}$, provided that the underlying distribution over the instance space is the products of log-concave distribution on the real line.}
\label{figure: isoperimetry}
\end{figure}

\subsubsection{Proofs of Theorem \ref{theorem: vulnerability}}

For a Borel set $\cA$ and for $\epsilon>0$, denote by $\cA_\epsilon=\{\x:d(\x,\cA)\le\epsilon\}$. The boundary measure of $\cA$ is then defined as
\begin{equation*}
\mu^+(\cA)=\liminf_{\epsilon\rightarrow+0}\frac{\mu(\cA_\epsilon)-\mu(\cA)}{\epsilon}.
\end{equation*}
The isoperimetric function is
\begin{equation}
\label{equ: isoperimetric problem}
I_\mu=\inf\{\mu^+(\cA):\mu(\cA)=1/2\}.
\end{equation}

Before proceeding, we cite the following results from \cite{barthe2001extremal}.
\begin{lemma}[Theorem 9, \cite{barthe2001extremal}]
\label{lemma: isoperimetry}
Let $\mu$ be an absolutely continuous log-concave probability measure on $\R$ with even density function. Denote by $d\mu=e^{-M(x)}$, where $M:\R\rightarrow[0,\infty]$ is convex. Assume that $M(0)=0$. If $\sqrt{M(x)}$ is a convex function, then for every integer $d$, we have $I_{\mu^{\otimes d}}\ge I_{\gamma^{\otimes d}}$, where $\gamma$ is the Gaussian measure with mean $0$ and variance $1/(2\pi)$. In particular, among sets of measure $1/2$ for $\mu^{\otimes d}$, the halfspace $[0,\infty)\times \R^{d-1}$ is solution to the isoperimetric problem \eqref{equ: isoperimetric problem}.
\end{lemma}

Now we are ready to prove Theorem \ref{theorem: vulnerability}.

\begin{proof}
We note that
\begin{equation*}
\begin{split}
&\quad\Pr[\X\in\bbB(\boundary(f),\epsilon)]\\
&=\Pr[\X\in\bbB(\boundary(f),\epsilon),\sign(f(\X))=+1]+\Pr[\X\in\bbB(\boundary(f),\epsilon),\sign(f(\X))=-1].
\end{split}
\end{equation*}
To apply Lemma \ref{lemma: isoperimetry}, we set the $\cA$ in Lemma \ref{lemma: isoperimetry} as the event $\{\sign(f(\X))=+1\}$. Therefore, the set
\begin{equation*}
\cA_\epsilon=\{\X\in\bbB(\boundary(f),\epsilon),\sign(f(\X))=-1\}.
\end{equation*}
By Lemma \ref{lemma: isoperimetry}, we know that for linear classifier $f_0$ which represents the halfspace $[0,\infty)\times \R^{d-1}$, and any classifier $f$,
\begin{equation}
\label{equ: middle 1}
\begin{split}
&\liminf_{\epsilon\rightarrow+0} \frac{\Pr_{\X\sim \mu^{\otimes d}}[\X\in\bbB(\boundary(f),\epsilon),\sign(f(\X))=-1]-\Pr[\sign(f(\X))=+1]}{\epsilon}\\
&\ge \liminf_{\epsilon\rightarrow+0} \frac{\Pr_{\X\sim \gamma^{\otimes d}}[\X\in\bbB(\boundary(f_0),\epsilon),\sign(f_0(\X))=-1]-\Pr[\sign(f_0(\X))=+1]}{\epsilon}.
\end{split}
\end{equation}
Similarly, we have
\begin{equation}
\label{equ: middle 2}
\begin{split}
&\liminf_{\epsilon\rightarrow+0} \frac{\Pr_{\X\sim \mu^{\otimes d}}[\X\in\bbB(\boundary(f),\epsilon),\sign(f(\X))=+1]-\Pr[\sign(f(\X))=-1]}{\epsilon}\\
&\ge \liminf_{\epsilon\rightarrow+0} \frac{\Pr_{\X\sim \gamma^{\otimes d}}[\X\in\bbB(\boundary(f_0),\epsilon),\sign(f_0(\X))=+1]-\Pr[\sign(f_0(\X))=-1]}{\epsilon}.
\end{split}
\end{equation}
Adding both sides of Eqns. \eqref{equ: middle 1} and \eqref{equ: middle 2}, we have
\begin{equation*}
\begin{split}
\liminf_{\epsilon\rightarrow+0} \frac{\Pr_{\X\sim \mu^{\otimes d}}[\X\in\bbB(\boundary(f),\epsilon)]}{\epsilon}&\ge \liminf_{\epsilon\rightarrow+0} \frac{\Pr_{\X\sim \gamma^{\otimes d}}[\X\in\bbB(\boundary(f_0),\epsilon)]}{\epsilon}\ge c.
\end{split}
\end{equation*}
\end{proof}

\subsection{Margin based generalization bounds}

Before proceeding, we first cite a useful lemma. We say that function $f_1:\R\rightarrow\R$ and $f_2:\R\rightarrow\R$ have a $\gamma$ separator if there exists a function $f_3:\R\rightarrow\R$ such that $|h_1-h_2|\le\gamma$ implies $f_1(h_1)\le f_3(h_2)\le f_2(h_1)$. For any given function $f_1$ and $\gamma>0$, one can always construct $f_2$ and $f_3$ such that $f_1$ and $f_2$ have a $\gamma$-separator $f_3$ by setting $f_2(h)=\sup_{|h-h'|\le2\gamma}f_1(h')$ and $f_3(h)=\sup_{|h-h'|\le\gamma}f_1(h')$.

\begin{lemma}[Corollary 1, \cite{zhang2002covering}]
\label{lemma: general margin bound}
Let $f_1$ be a function $\R\rightarrow\R$. Consider a family of functions $f_2^\gamma:\R\rightarrow\R$, parametrized by $\gamma$, such that $0\le f_1\le f_2^\gamma\le 1$. Assume that for all $\gamma$, $f_1$ and $f_2^\gamma$ has a $\gamma$ separator. Assume also that $f_2^\gamma(z)\ge f_2^{\gamma'}(z)$ when $\gamma\ge\gamma'$. Let $\gamma_1>\gamma_2>...$ be a decreasing sequence of parameters, and $p_i$ be a sequence of positive numbers such that $\sum_{i=1}^\infty p_i=1$, then for all $\eta>0$, with probability of at least $1-\delta$ over data:
\begin{equation*}
\bbE_{(\X,Y)\sim\cD} f_1(\cL(\w,\X,Y))\le \frac{1}{n}\sum_{i=1}^n f_2^\gamma(\cL(\w,\x_i,y_i))+\sqrt{\frac{32}{n}\left(\ln 4\cN_\infty(\cL,\gamma_i,\x_{1:n})+\ln\frac{1}{p_i \delta}\right)}
\end{equation*}
for all $\w$ and $\gamma$, where for each fixed $\gamma$, we use $i$ to denote the smallest index such that $\gamma_i\le\gamma$.
\end{lemma}

\begin{lemma}[Theorem 4, \cite{zhang2002covering}]
\label{lemma: covering number}
If $\|\x\|_p\le b$ and $\|\w\|_q\le a$, where $2\le p<\infty$ and $1/p+1/q=1$, then $\forall\gamma>0$,
\begin{equation*}
\log_2 \cN_\infty(\cL,\gamma,n)\le 36(p-1)\frac{a^2b^2}{\gamma^2}\log_2[2\lceil 4ab/\gamma+2\rceil+1].
\end{equation*}
\end{lemma}

\begin{theorem}
Suppose that the data is 2-norm bounded by $\|\x\|_2\le b$. Consider the family $\Gamma$ of linear classifier $\w$ with $\|\w\|_2=1$. Let $\cR_\adv(\w):=\bbE_{(\X,Y)\sim\cD}\1[\exists \X^\adv\in\bbB_2(\X,\epsilon) \text{ such that } Y\w^T\X^\adv\le 0]$. Then with probability at least $1-\delta$ over $n$ random samples $(\x_i,y_i)\sim\cD$, for all margin width $\gamma>0$ and $\w\in\Gamma$, we have
\begin{equation*}
\cR_\adv(\w)\le \frac{1}{n}\sum_{i=1}^n\1(\exists\x_i^\adv\in\bbB(\x_i,\epsilon)\text{ s.t. }y_i\w^T\x_i^\adv\le2\gamma)+\sqrt{\frac{C}{n}\left(\frac{b^2}{\gamma^2}\ln n+\ln\frac{1}{\delta}\right)}.
\end{equation*}
\end{theorem}

\begin{proof}
The theorem is a straightforward result of Lemmas \ref{lemma: general margin bound} and \ref{lemma: covering number} with
$$\cL(\w,\x,y)=\min_{\x^\adv\in\bbB(\x,\epsilon)}y\w^T\x^\adv,$$
$$f_1(g)=\1(g\le 0)\quad\mbox{and}\quad f_2^\gamma(h)=\sup_{|g-h|<2\gamma} f_1(g)=f_1(g-2\gamma)=\1(g\le 2\gamma),$$
and $\gamma_i=b/2^i$ and $p_i=1/2^i$.
\end{proof}

We note that for the $\ell_2$ ball $\bbB_2(\x,\epsilon)=\{\x':\|\x-\x'\|_2\le\epsilon\}$, we have
\begin{equation*}
\1(\exists\x_i^\adv\in\bbB(\x_i,\epsilon)\text{ s.t. }y_i\w^T\x_i^\adv\le2\gamma)=\max_{\x_i^\adv\in\bbB(\x_i,\epsilon)}\1(y_i\w^T\x_i^\adv\le2\gamma)=\1(y_i\w^T\x_i\le2\gamma+\epsilon).
\end{equation*}
Therefore, we can design the following algorithm---Algorithm \ref{algorithm: adversarial training of linear separator via structural risk minimization}.

\begin{algorithm}[ht]
\caption{Adversarial Training of Linear Separator via Structural Risk Minimization}
\begin{algorithmic}
\label{algorithm: adversarial training of linear separator via structural risk minimization}
\STATE {\bfseries Input:} Samples $(\x_{1:n},y_{1:n})\sim\cD$, a bunch of margin parameters $\gamma_1,...,\gamma_T$.
\STATE {\bfseries 1:} \textbf{For} $k=1,2,...,T$
\STATE {\bfseries 2:} \quad Solve the minimax optimization problem:
\begin{equation*}
\begin{split}
\cL_k(\w_k^*,\x_{1:n},y_{1:n})&=\min_{\w\in\bbS(\0,1)} \frac{1}{n}\sum_{i=1}^n \max_{\x_i^\adv\in\bbB(\x_i,\epsilon)}\1(y_i\w^T\x_i^\adv\le2\gamma_k)\\
&=\min_{\w\in\bbS(\0,1)} \frac{1}{n}\sum_{i=1}^n \1(y_i\w^T\x_i\le2\gamma_k+\epsilon).
\end{split}
\end{equation*}
\STATE {\bfseries 3:} \textbf{End For}
\STATE {\bfseries 4:} $k^*=\argmin_{k} \cL_k(\w_k^*,\x_{1:n},y_{1:n})+\sqrt{\frac{C}{n}\left(\frac{b^2}{\gamma_k^2}\ln n+\ln\frac{1}{\delta}\right)}$.
\STATE {\bfseries Output:} Hypothesis $\w_{k^*}$.
\end{algorithmic}
\end{algorithm}

\subsection{A lemma}

We denote by $f^*(\cdot):=2\eta(\cdot)-1$ the Bayes decision rule throughout the proofs.

\begin{lemma}
\label{lemma: risk equality}
For any classifier $f$, we have
\begin{equation*}
\begin{split}
\cR_\adv(f)-\cR_\nat^*=&\bbE[\1\{\sign(f(\X))\not=\sign(f^*(\X)),\X\in\bbB(\boundary(f),\epsilon)^\perp\}|2\eta(\X)-1|]\\
&+\Pr[\X\in \bbB(\boundary(f),\epsilon),\sign(f^*(\X))=Y].
\end{split}
\end{equation*}
\end{lemma}

\begin{proof}
For any classifier $f$, we have
\begin{equation*}
\begin{split}
&\quad \Pr(\exists \X'\in\bbB(\X,\epsilon) \text{ s.t. } \sign(f(\X'))\not=Y|\X=\x)\\
&=\Pr(Y=1,\exists \X'\in\bbB(\X,\epsilon) \text{ s.t. }\sign(f(\X'))=-1|\X=\x)\\
&\quad+\Pr(Y=-1,\exists \X'\in\bbB(\X,\epsilon) \text{ s.t. }\sign(f(\X'))=1|\X=\x)\\
&=\bbE[\1\{Y=1\}\1\{\exists \X'\in\bbB(\X,\epsilon) \text{ s.t. }\sign(f(\X'))=-1\}|\X=\x]\\
&\quad +\bbE[\1\{Y=-1\}\1\{\exists \X'\in\bbB(\X,\epsilon) \text{ s.t. }\sign(f(\X'))=1\}|\X=\x]\\
&=\1\{\exists \x'\in\bbB(\x,\epsilon) \text{ s.t. }\sign(f(\x'))=-1\}\bbE\1\{Y=1|\X=\x\}\\
&\quad+\1\{\exists \x'\in\bbB(\x,\epsilon) \text{ s.t. }\sign(f(\x'))=1\}\bbE\1\{Y=-1|\X=\x\}\\
&=\1\{\exists \x'\in\bbB(\x,\epsilon) \text{ s.t. }\sign(f(\x'))=-1\}\eta(\x)+\1\{\exists \x'\in\bbB(\x,\epsilon) \text{ s.t. }\sign(f(\x'))=1\}(1-\eta(\x))\\
&=
\begin{cases}
1, & \x\in \bbB(\boundary(f),\epsilon),\\
\1\{\sign(f(\x))=-1\}(2\eta(\x)-1)+(1-\eta(\x)), & \text{otherwise}.
\end{cases}
\end{split}
\end{equation*}

Therefore,
\begin{equation*}
\begin{split}
&\quad \cR_\adv(f)\\
&=\int_{\cX} \Pr[\exists \X'\in\bbB(\X,\epsilon) \text{ s.t. }\sign(f(\X'))\not=Y|\X=\x] d \Pr\nolimits_{\X}(\x)\\
&=\int_{\bbB(\boundary(f),\epsilon)} \Pr[\exists \X'\in\bbB(\X,\epsilon) \text{ s.t. }\sign(f(\X'))\not=Y|\X=\x] d \Pr\nolimits_{\X}(\x)\\
&\quad +\int_{\bbB(\boundary(f),\epsilon)^\perp} \Pr[\exists \X'\in\bbB(\X,\epsilon) \text{ s.t. }\sign(f(\X'))\not=Y|\X=\x] d \Pr\nolimits_{\X}(\x)\\
&=\Pr(\X\in \bbB(\boundary(f),\epsilon))\\
&\quad +\int_{\bbB(\boundary(f),\epsilon)^\perp} [\1\{\sign(f(\x))=-1\}(2\eta(\x)-1)+(1-\eta(\x))]d \Pr\nolimits_{\X}(\x).
\end{split}
\end{equation*}
We have
\begin{equation*}
\begin{split}
&\quad\cR_\adv(f)-\cR_\nat(f^*)\\
&=\Pr(\X\in \bbB(\boundary(f),\epsilon))+\int_{\bbB(\boundary(f),\epsilon)^\perp} [\1\{\sign(f(\x))=-1\}(2\eta(\x)-1)+(1-\eta(\x))]d \Pr\nolimits_{\X}(\x)\\
&\quad-\int_{\bbB(\boundary(f),\epsilon)^\perp} [\1\{\sign(f^*(\x))=-1\}(2\eta(\x)-1)+(1-\eta(\x))]d \Pr\nolimits_{\X}(\x)\\
&\quad-\int_{\bbB(\boundary(f),\epsilon)} [\1\{\sign(f^*(\x))=-1\}(2\eta(\x)-1)+(1-\eta(\x))]d \Pr\nolimits_{\X}(\x)\\
&=\Pr(\X\in \bbB(\boundary(f),\epsilon))-\int_{\bbB(\boundary(f),\epsilon)} [\1\{\sign(f^*(\x))=-1\}(2\eta(\x)-1)+(1-\eta(\x))]d \Pr\nolimits_{\X}(\x)\\
&\quad+\bbE[\1\{\sign(f(\X))\not=\sign(\eta(\X)-1/2),\X\in\bbB(\boundary(f),\epsilon)^\perp\}|2\eta(\X)-1|]\\
&=\Pr(\X\in \bbB(\boundary(f),\epsilon))-\bbE[\1\{\X\in \bbB(\boundary(f),\epsilon)\}\min\{\eta(\X),1-\eta(\X)\}]\\
&\quad+\bbE[\1\{\sign(f(\X))\not=\sign(\eta(\X)-1/2),\X\in\bbB(\boundary(f),\epsilon)^\perp\}|2\eta(\X)-1|]\\
&=\bbE[\1\{\X\in \bbB(\boundary(f),\epsilon)\}\max\{\eta(\X),1-\eta(\X)\}]\\
&\quad +\bbE[\1\{\sign(f(\X))\not=\sign(\eta(\X)-1/2),\X\in\bbB(\boundary(f),\epsilon)^\perp\}|2\eta(\X)-1|]\\
&=\Pr[\X\in \bbB(\boundary(f),\epsilon),\sign(f^*(\X))=Y]\\
&\quad +\bbE[\1\{\sign(f(\X))\not=\sign(f^*(\X)),\X\in\bbB(\boundary(f),\epsilon)^\perp\}|2\eta(\X)-1|].
\end{split}
\end{equation*}
\end{proof}

\section{Extra Experimental Results}
In this section, we provide extra experimental results to verify the effectiveness of our proposed method TRADES.
\subsection{Experimental setup in Section~\ref{subsec:white-box}}
We use the same model, i.e., the WRN-34-10 architecture in~\cite{zagoruyko2016wide}, to implement the methods in \cite{zheng2016improving}, \cite{kurakin2016adversarial} and \cite{ross2017improving}. The experimental setup is the same as TRADES, which is specified in the beginning of Section~\ref{sec:experiments}. For example, we use the same batch size and learning rate for all the methods. More specifically, we find that using one-step adversarial perturbation method like FGSM in the regularization term, defined in \cite{kurakin2016adversarial}, cannot defend against FGSM$^k$ (white-box) attack. Therefore, we use FGSM$^k$ with the cross-entropy loss to calculate the adversarial example $\X'$ in the regularization term, and the perturbation step size $\eta_1$ and number of iterations $K$ are the same as in the beginning of Section~\ref{sec:experiments}. 

As for defense models in Table~\ref{table: icml best paper defense}, we implement the `TRADES' models, the models trained by using other regularization losses in \cite{kurakin2016adversarial,ross2017improving,zheng2016improving}, and the defense model `Madry' in the antepenultimate line of Table~\ref{table: icml best paper defense}. We evaluate \cite{wong1805scaling}'s model based on the checkpoint provided by the authors. The rest of the models in Table~\ref{table: icml best paper defense} are reported in \cite{athalye2018obfuscated}.

\subsection{Extra attack results in Section~\ref{subsec:white-box}}

Extra white-box attack results are provided in Table \ref{table: supplementary-icml best paper defense}.

\begin{table*}[ht]
	\caption{Results of our method TRADES under different white-box attacks.}
	\label{table: supplementary-icml best paper defense}
	\centering
	\begin{tabular}{c||c|c|c|c|c}%
		\hline
		Defense & Under which attack & Dataset & Distance & $\cA_\nat(f)$   & $\cA_\adv(f)$  
		\\
		\hline
		{TRADES} ($1/\lambda=1.0$) & FGSM  & CIFAR10 &  $0.031$ ($\ell_\infty$) & 88.64\% & 56.38\% \\
		{TRADES} ($1/\lambda=1.0$) & DeepFool ($\ell_2$)  & CIFAR10 &  $0.031$ ($\ell_\infty$) & 88.64\% & 84.49\% \\
		\hline
		{TRADES} ($1/\lambda=6.0$) & FGSM  & CIFAR10 &  $0.031$ ($\ell_\infty$) & 84.92\% & 61.06\% \\
		{TRADES} ($1/\lambda=6.0$)  & DeepFool ($\ell_2$) & CIFAR10 &  $0.031$ ($\ell_\infty$) & 84.92\% & 81.55\% \\
		\hline
	\end{tabular}
\end{table*}

The attacks in Table~\ref{table: icml best paper defense} and Table~\ref{table: supplementary-icml best paper defense} include FGSM$^k$~\cite{kurakin2016adversarial}, DeepFool ($\ell_\infty$)~\cite{moosavi2016deepfool},  LBFGSAttack~\cite{tabacof2016exploring}, MI-FGSM~\cite{dong2018boosting}, C\&W~\cite{carlini2017towards}, FGSM~\cite{kurakin2016adversarial}, and DeepFool ($\ell_2$)~\cite{moosavi2016deepfool}.

\subsection{Extra attack results in Section~\ref{subsec:black-box}}
Extra black-box attack results are provided in Table \ref{table: MNIST FGSM black-box defense} and Table \ref{table: CIFAR10 FGSM black-box defense}. We apply black-box FGSM attack on the MNIST dataset and the CIFAR10 dataset.

\begin{table*}[ht]
	\caption{Comparisons of TRADES with prior defense models under black-box FGSM attack on the MNIST dataset. The models inside parentheses are source models which provide gradients to adversarial attackers.}
	\label{table: MNIST FGSM black-box defense}
	\centering
	\begin{tabular}{c||c|c}%
		\hline
		Defense Model & \multicolumn{2}{c}{Robust Accuracy $\cA_\adv(f)$}
		\\
		\hline
		{Madry} & 97.68\% (Natural)  & 98.11\% (Ours) \\ 
		\hline
		{TRADES} & \textbf{97.75\%} (Natural) & \textbf{98.44\%} (Madry)  \\
		\hline
	\end{tabular}
\end{table*}

\begin{table*}[ht]
	\caption{Comparisons of TRADES with prior defense models under black-box FGSM attack on the CIFAR10 dataset. The models inside parentheses are source models which provide gradients to adversarial attackers.}
	\label{table: CIFAR10 FGSM black-box defense}
	\centering
	\begin{tabular}{c||c|c}%
		\hline
		Defense Model & \multicolumn{2}{c}{Robust Accuracy $\cA_\adv(f)$}
		\\
		\hline
		{Madry} & 84.02\% (Natural)  & 67.66\% (Ours) \\ 
		\hline
		{TRADES} & \textbf{86.84\%} (Natural) & \textbf{71.52}\% (Madry)  \\
		\hline
	\end{tabular}
\end{table*}

\subsection{Experimental setup in Section~\ref{subsec:black-box}}
The robust accuracy of~\cite{madry2018towards}'s CNN model is $96.01\%$ on the MNIST dataset. The robust accuracy of~\cite{madry2018towards}'s WRN-34-10 model is $47.66\%$ on the CIFAR10 dataset. Note that we use the same white-box attack method introduced in the Section \ref{subsec:white-box}, i.e., FGSM$^{20}$,  to evaluate the robust accuracies of Madry's models.

\subsection{Interpretability of the robust models trained by TRADES}

\subsubsection{Adversarial examples on MNIST and CIFAR10 datasets}
In this section, we provide adversarial examples on MNIST and CIFAR10. We apply \textbf{\texttt{foolbox}}\footnote{Link: \url{https://foolbox.readthedocs.io/en/latest/index.html}}~\cite{rauber2017foolbox} to generate adversarial examples, which is able to return the smallest adversarial perturbations under the $\ell_\infty$ norm distance. The adversarial examples are generated by using FGSM$^{k}$ (white-box) attack on the models described in Section~\ref{sec:experiments}, including $\textit{Natural}$ models, $\textit{Madry}$'s models and $\textit{TRADES}$ models. Note that the FGSM$^{k}$ attack is \texttt{foolbox.attacks.LinfinityBasicIterativeAttack} in \textbf{\texttt{foolbox}}. See Figure~\ref{figure: MNIST-adv-images} and Figure~\ref{figure: CIFAR10-adv-images} for the adversarial examples of different models on  MNIST and CIFAR10 datasets.

\subsubsection{Adversarial examples on Bird-or-Bicycle dataset}
We find that the robust models trained by TRADES have strong interpretability. To see this, we apply a (spatial-tranformation-invariant) version of TRADES to train ResNet-50 models in response to the unrestricted adversarial examples of the bird-or-bicycle dataset~\cite{brown2018unrestricted}. The dataset is \emph{bird-or-bicycle}, which consists of ~30,000 pixel-$224\times 224$ images with label either `bird' or `bicycle'. The unrestricted threat models include structural perturbations, rotations, translations, resizing, 17+ common corruptions, etc.

We show in Figures \ref{figure: interpretability bike} and \ref{figure: interpretability bird} the adversarial examples by the boundary attack with random spatial transformation on our robust model trained by the variant of TRADES. The boundary attack~\cite{brendel2017decision} is a black-box attack method which searches for data points near the decision boundary and attack robust models by these data points. Therefore, the adversarial images obtained by boundary attack characterize the images around the decision boundary of robust models. We attack our model by boundary attack with random spatial transformations, a baseline in the competition. The classification accuracy on the adversarial test data is as high as $95\%$ (at $80\%$ coverage), even though the adversarial corruptions are perceptible to human. We observe that the robust model trained by TRADES has strong interpretability: in Figure \ref{figure: interpretability bike} all of adversarial images have obvious feature of `bird', while in Figure \ref{figure: interpretability bird} all of adversarial images have obvious feature of `bicycle'. This shows that images around the decision boundary of truly robust model have features of both classes.

\newpage

\begin{figure}
	\centering
	\subfigure[adversarial examples of class `3']{
		\fbox{\includegraphics[width=7.5cm]{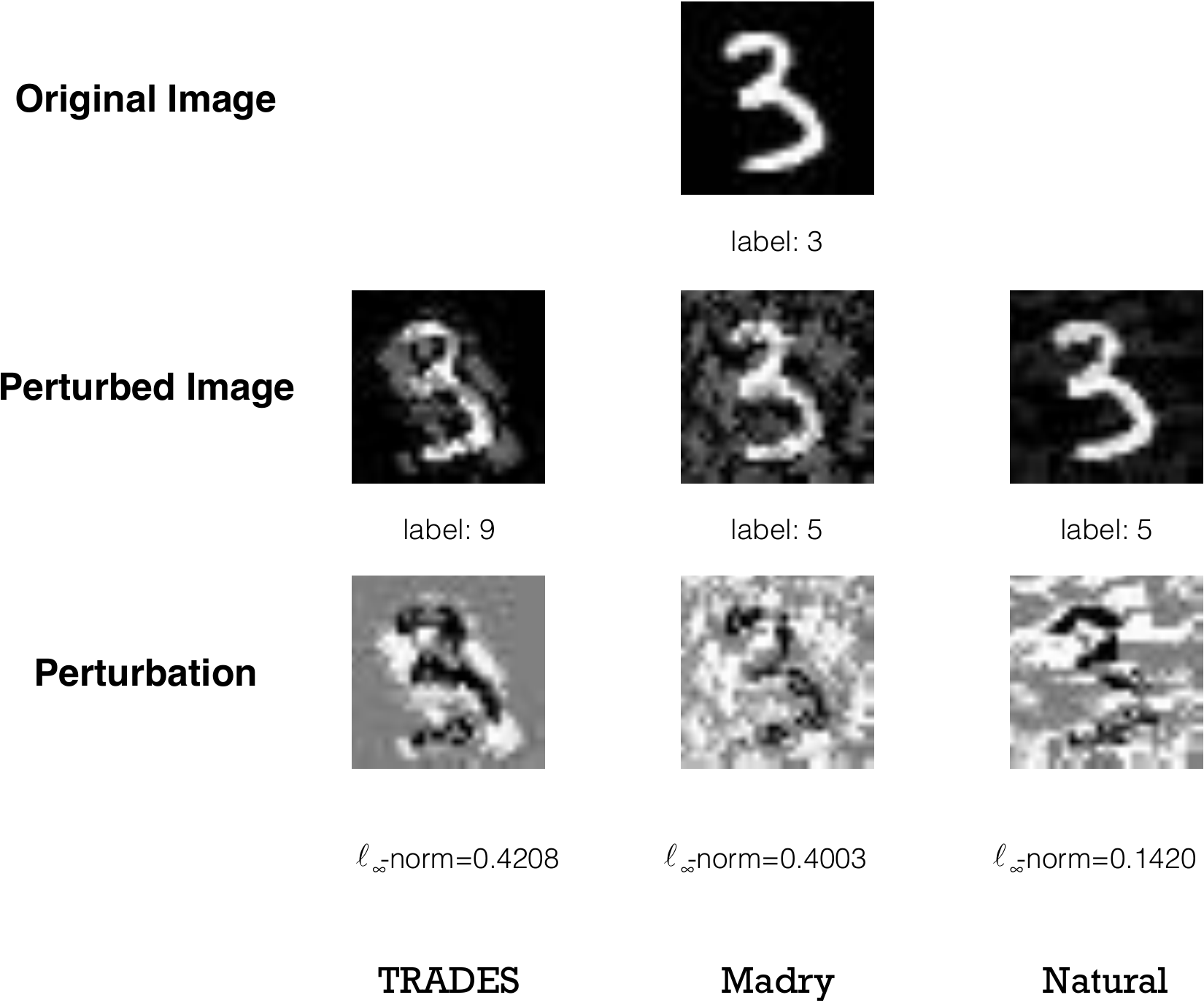}}
	}
	\quad
	\subfigure[adversarial examples of class `4']{
		\fbox{\includegraphics[width=7.5cm]{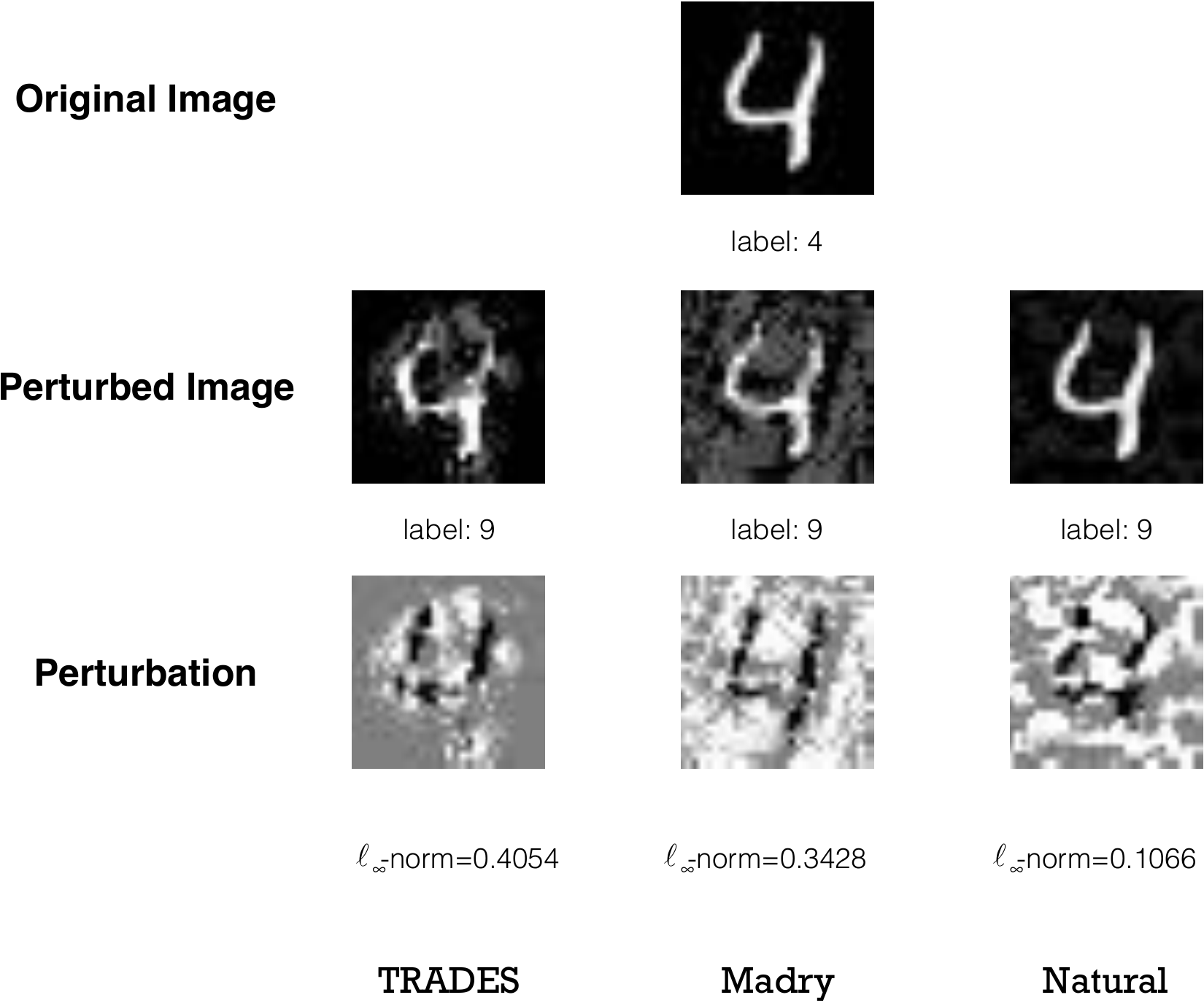}}
	}
	\quad
	\subfigure[adversarial examples of class `5']{
		\fbox{\includegraphics[width=7.5cm]{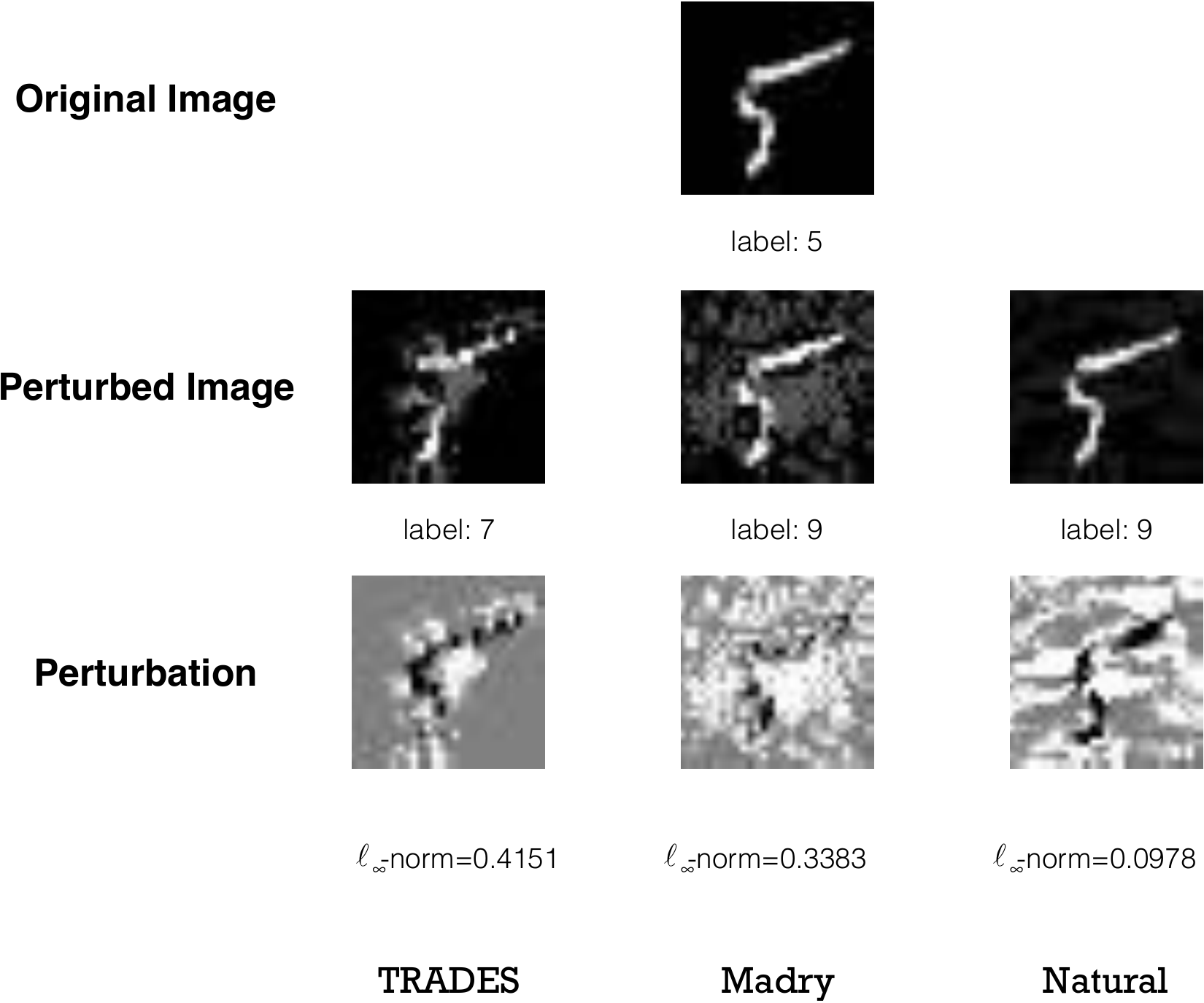}}
	}
	\quad
	\subfigure[adversarial examples of class `7']{
		\fbox{\includegraphics[width=7.5cm]{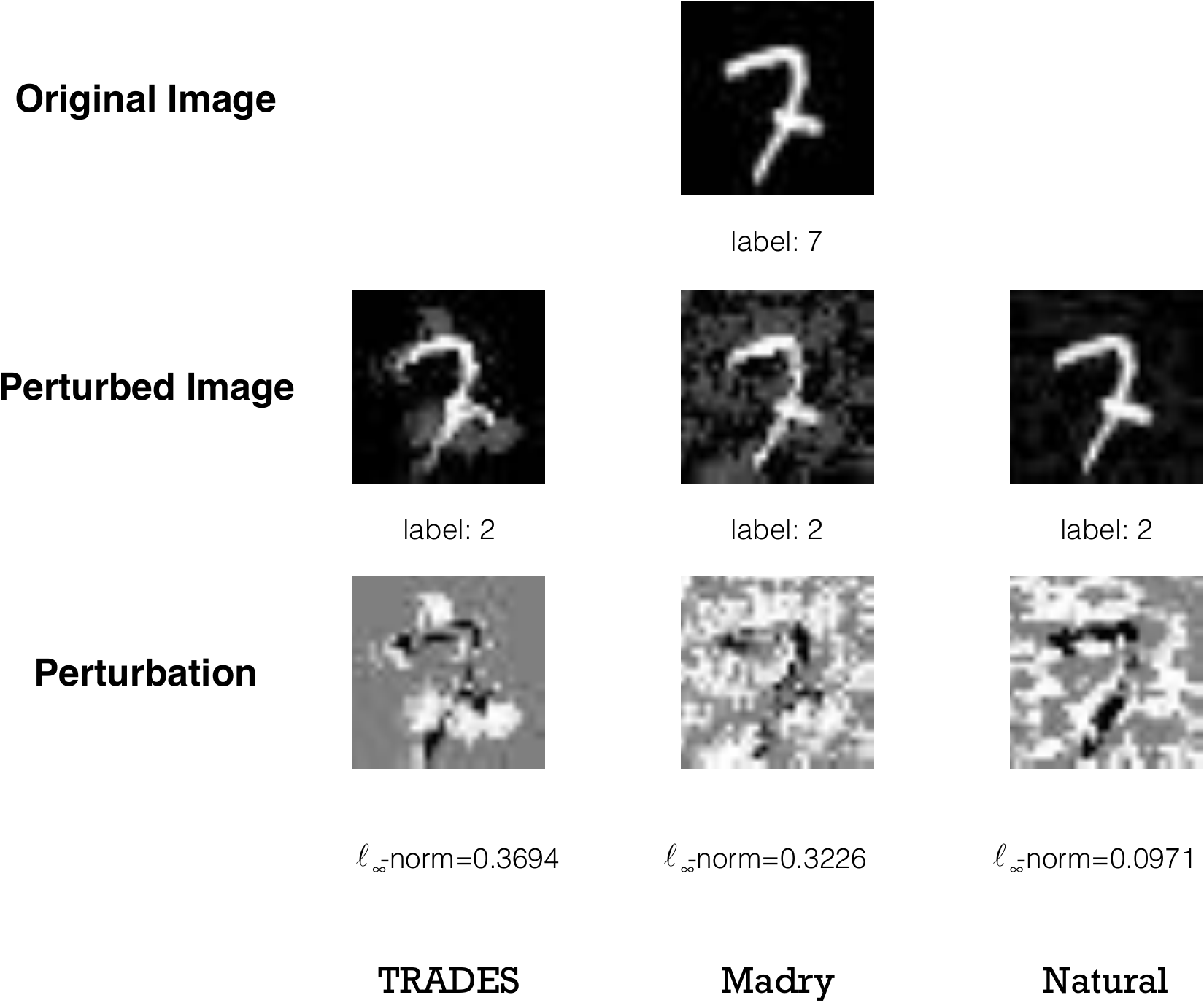}}
	}
	\caption{Adversarial examples on MNIST dataset. In each subfigure, the image in the first row is the original image and we list the corresponding correct label beneath the image. We show the perturbed images in the second row. The differences between the perturbed images and the original images, i.e., the perturbations, are shown in the third row. In each column, the perturbed image and the perturbation are generated by FGSM$^{k}$ (white-box) attack on the model listed below. The labels beneath the perturbed images are the predictions of the corresponding models, which are different from the correct labels. We record the smallest perturbations in terms of $\ell_\infty$ norm that make the models predict a wrong label.}
	\label{figure: MNIST-adv-images}
\end{figure}

\begin{figure}[htbp]
	\centering
	\subfigure[adversarial examples of class `automobile']{
		\fbox{\includegraphics[width=7.5cm]{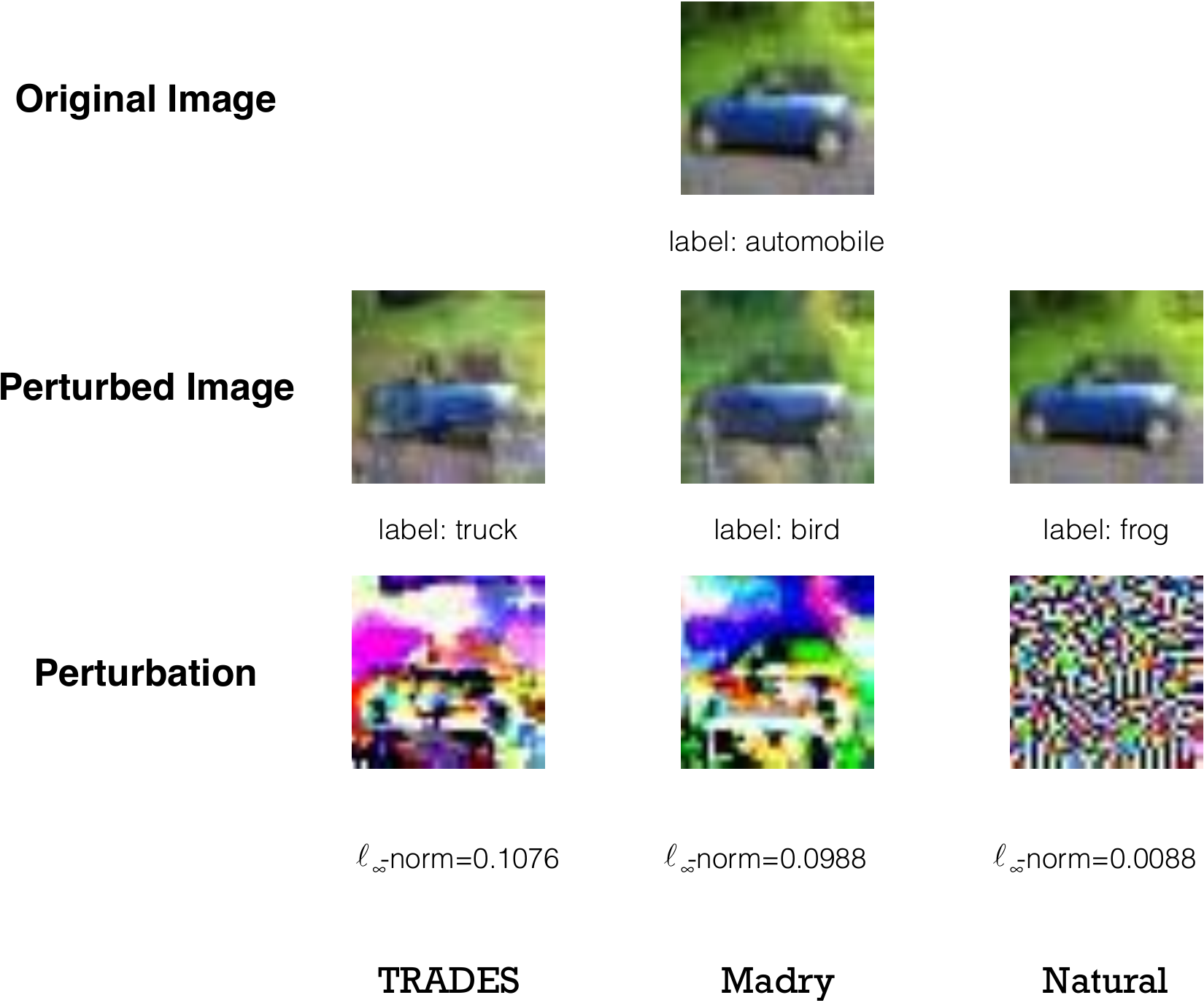}}
	}
	\quad
	\subfigure[adversarial examples of class `deer']{
		\fbox{\includegraphics[width=7.5cm]{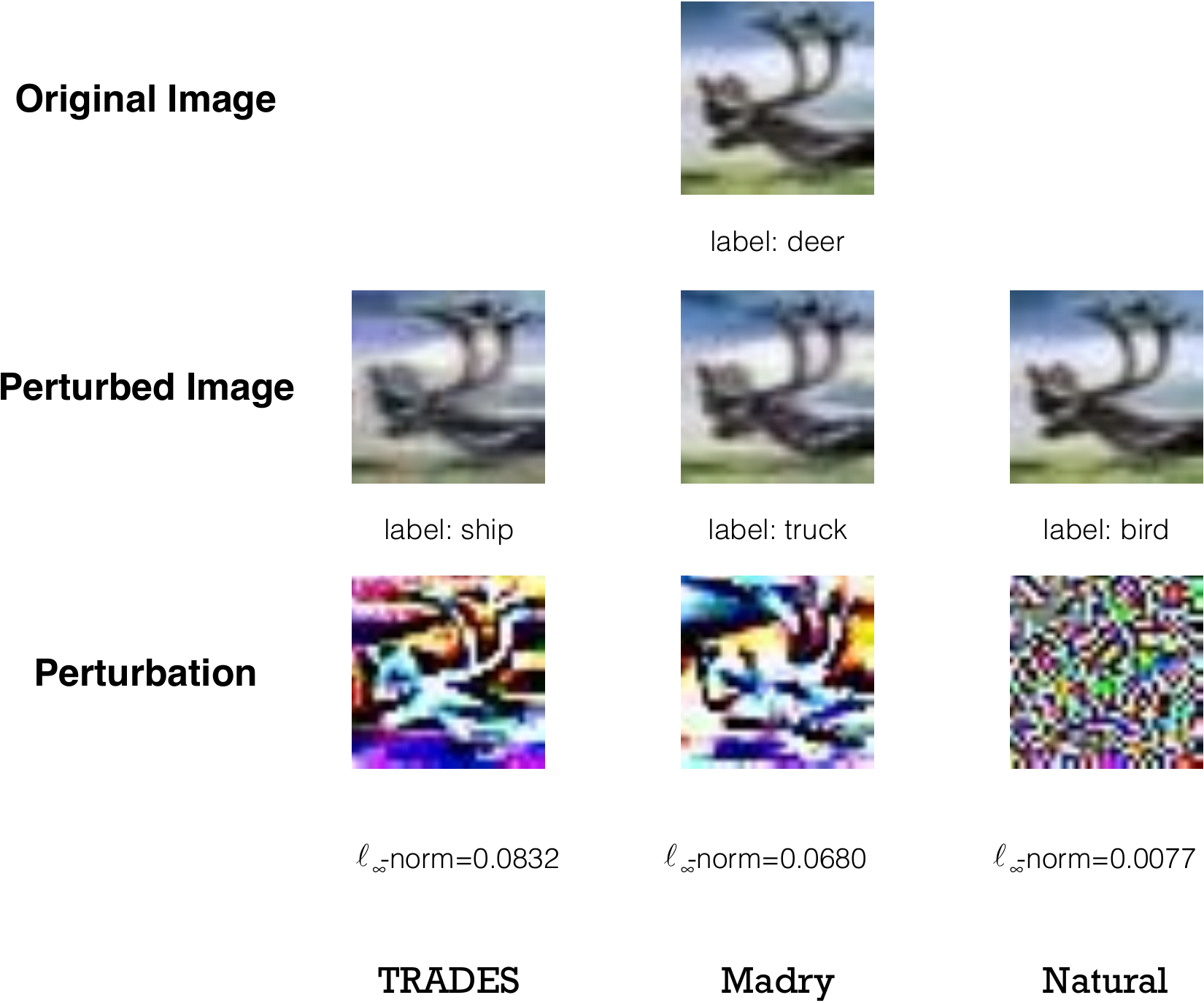}}
	}
	\quad
	\subfigure[adversarial examples of class `ship']{
		\fbox{\includegraphics[width=7.5cm]{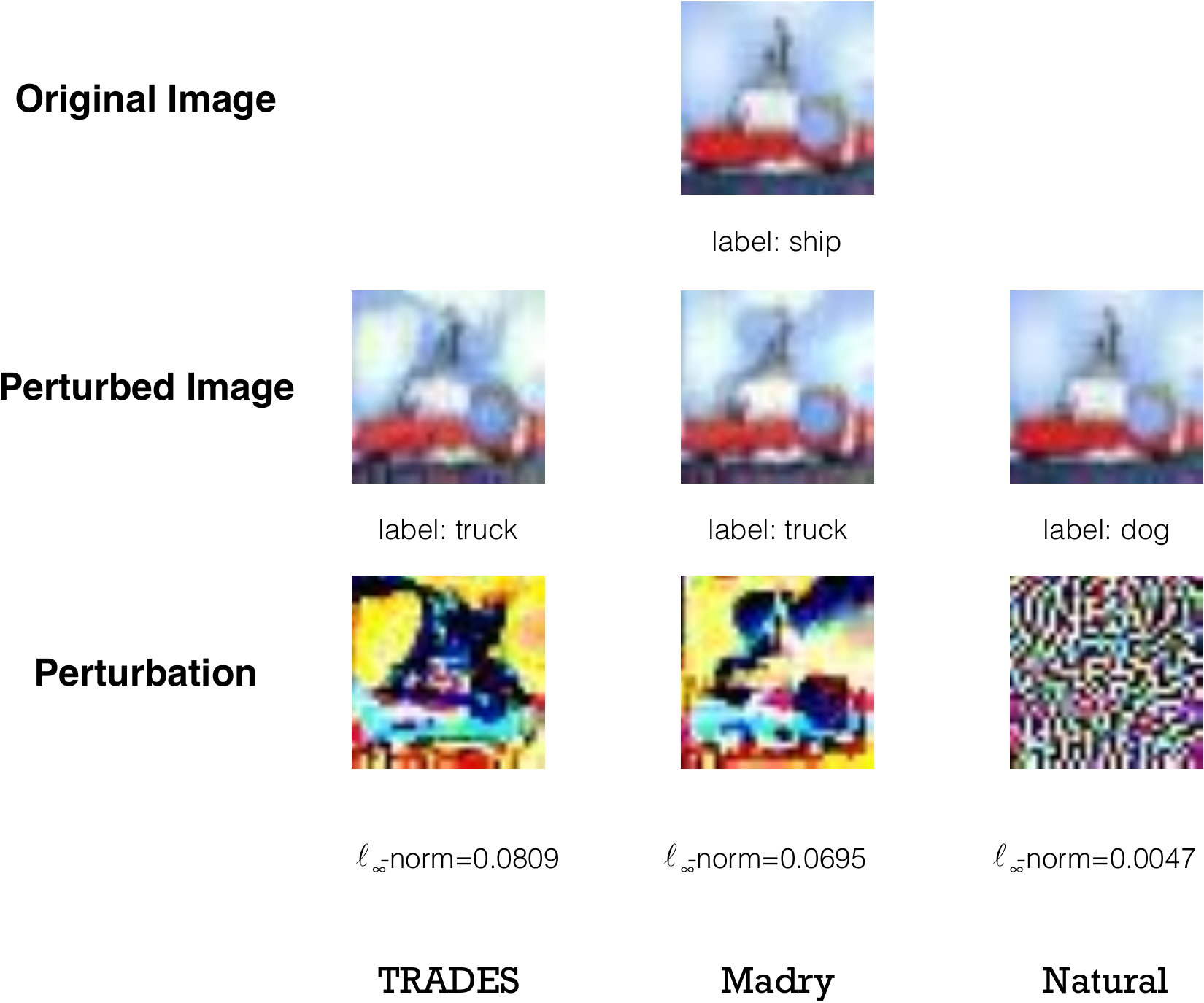}}
	}
	\quad
	\subfigure[adversarial examples of class `truck']{
		\fbox{\includegraphics[width=7.5cm]{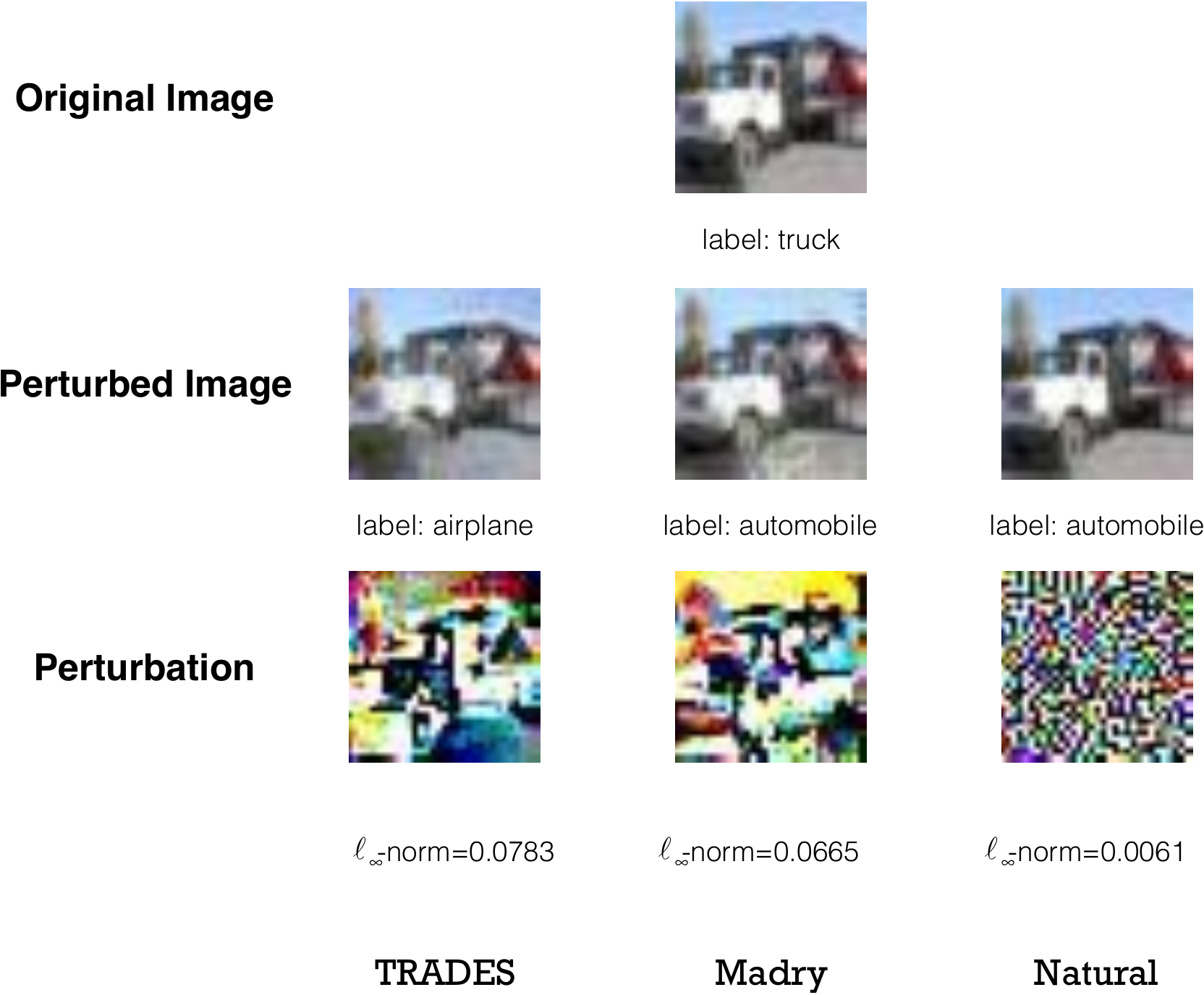}}
	}
	\caption{Adversarial examples on CIFAR10 dataset. In each subfigure, the image in the first row is the original image and we list the corresponding correct label beneath the image. We show the perturbed images in the second row. The differences between the perturbed images and the original images, i.e., the perturbations, are shown in the third row. In each column, the perturbed image and the perturbation are generated by FGSM$^{k}$ (white-box) attack on the model listed below. The labels beneath the perturbed images are the predictions of the corresponding models, which are different from the correct labels. We record the smallest perturbations in terms of $\ell_\infty$ norm that make the models predict a wrong label \textbf{(best viewed in color)}.}
	\label{figure: CIFAR10-adv-images}
\end{figure}

\begin{figure}[htbp]
\centering
	\subfigure[clean example]{
	\includegraphics[width=5cm]{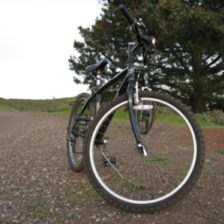}}
	\subfigure[adversarial example by boundary attack with random spatial transformation]{
	\includegraphics[width=5cm]{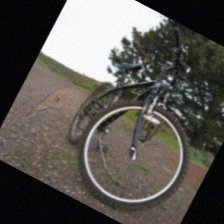}}\\
	\subfigure[clean example]{
	\includegraphics[width=5cm]{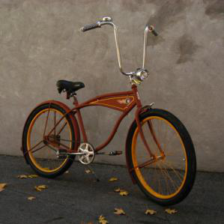}}
	\subfigure[adversarial example by boundary attack with random spatial transformation]{
	\includegraphics[width=5cm]{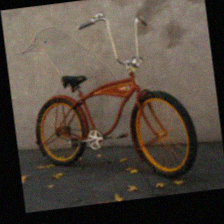}}\\
	\subfigure[clean example]{
	\includegraphics[width=5cm]{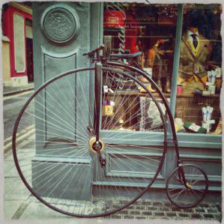}}
	\subfigure[adversarial example by boundary attack with random spatial transformation]{
	\includegraphics[width=5cm]{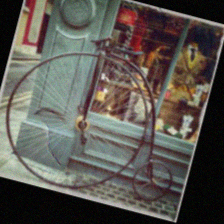}}
	\caption{Adversarial examples by boundary attack with random spatial transformation on the ResNet-50 model trained by a variant of TRADES. The ground-truth label is `bicycle', and our robust model recognizes the adversarial examples correctly as `bicycle'. It shows in the second column that all of adversarial images have obvious feature of `bird' \textbf{(best viewed in color)}.}
	\label{figure: interpretability bike}
\end{figure}

\begin{figure}[htbp]
\centering
	\subfigure[clean example]{
	\includegraphics[width=5cm]{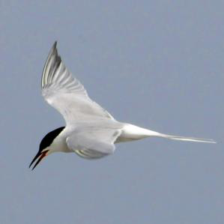}}
	\subfigure[adversarial example by boundary attack with random spatial transformation]{
	\includegraphics[width=5cm]{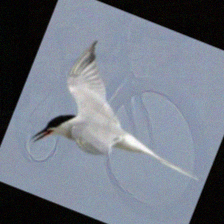}}\\
	\subfigure[clean example]{
	\includegraphics[width=5cm]{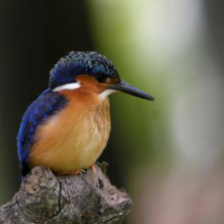}}
	\subfigure[adversarial example by boundary attack with random spatial transformation]{
	\includegraphics[width=5cm]{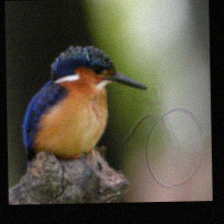}}\\
	\subfigure[clean example]{
	\includegraphics[width=5cm]{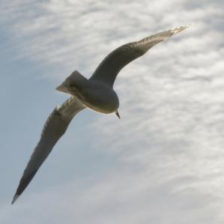}}
	\subfigure[adversarial example by boundary attack with random spatial transformation]{
	\includegraphics[width=5cm]{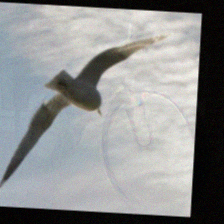}}
	\caption{Adversarial examples by boundary attack with random spatial transformation on the ResNet-50 model trained by a variant of TRADES. The ground-truth label is `bird', and our robust model recognizes the adversarial examples correctly as `bird'. It shows in the second column that all of adversarial images have obvious feature of `bicycle' \textbf{(best viewed in color)}.}
	\label{figure: interpretability bird}
\end{figure}

\end{document}